\documentclass[accepted]{uai2026} 
 
\usepackage[american]{babel}

\usepackage{natbib}
\usepackage{mathtools} 
\usepackage{amsthm, amssymb}
\usepackage{bm}
\usepackage{xparse}    

\usepackage{graphicx}
\usepackage{subcaption}
\usepackage{standalone}
\usepackage{tikz}
\usepackage{pgfplots}
\pgfplotsset{compat=1.18}

\usepackage{booktabs}
\usepackage{tabularx}
\usepackage{array}
\usepackage{adjustbox}

\usepackage[disable]{todonotes}

\usepackage{xurl} 
\usepackage[nameinlink]{cleveref}
 
\usepackage{algorithm}
\usepackage{algpseudocode}
\algrenewcommand\algorithmicrequire{\textbf{Input:}}
\algrenewcommand\algorithmicensure{\textbf{Output:}}
 
\newtheorem{theorem}{Theorem}
\newtheorem{lemma}[theorem]{Lemma}
\newtheorem{proposition}[theorem]{Proposition}
\newtheorem{corollary}[theorem]{Corollary}

\newtheorem{remark}[theorem]{Remark}
  
\newcommand{\ent}[2][]{%
  \mathbb{H}%
  \ifblank{#1}{\left[#2\right]}{\csname #1l\endcsname[#2\csname #1r\endcsname]}%
}

\newcommand{\hent}[2][]{%
  \mathrm{h} \!%
  \ifblank{#1}{\left[#2\right]}{\csname #1l\endcsname[#2\csname #1r\endcsname]}%
}

\newcommand{\xent}[3][]{%
  \mathbb{H}%
  \ifblank{#1}{\left[#2, #3\right]}{\csname #1l\endcsname[#2, #3\csname #1r\endcsname]}%
}
 
\newcommand{\KL}[3][]{%
  \mathbb{D}_{\mathrm{KL}}%
  \ifblank{#1}{\left[#2 \| #3 \right]}{\csname #1l\endcsname[#2 \| #3 \csname #1r\endcsname]}%
}
\newcommand{\Ex}[3][]{%
  \mathbb{E}_{#2}%
  \ifblank{#1}{\left[#3\right]}{\csname #1l\endcsname[#3\csname #1r\endcsname]}%
}

\newcommand{\dif}{\mathrm{d}}

\newcommand\lref[1]{\hyperref[#1]{Lemma~\ref*{#1}}}

\newcommand{\capsecref}[1]{\hyperref[#1]{Section~\ref*{#1}}}
\newcommand{\refappx}[1]{\hyperref[#1]{Appendix~\ref*{#1}}}
\newcommand{\refthm}[1]{\hyperref[#1]{Theorem~\ref*{#1}}}
\newcommand{\refprop}[1]{\hyperref[#1]{Proposition~\ref*{#1}}}

\title{What Type of Inference is Active Inference?}

\author[1,2]{\href{mailto:w.w.l.nuijten@tue.nl?Subject=Your UAI 2026 paper}{Wouter~W.~L.~Nuijten}{}}
\author[1]{Mykola~Lukashchuk}
\author[1]{Thijs~van~de~Laar}
\author[1,2]{Bert~de~Vries}

\affil[1]{%
    Department of Electrical Engineering\\
    Eindhoven University of Technology\\
    Eindhoven, the Netherlands
}
\affil[2]{%
    Lazy Dynamics\\
    Utrecht, the Netherlands
}

\begin{document}
\maketitle

\begin{abstract}
Active inference casts decision-making as inference, with the Expected Free Energy (EFE) unifying goal-directed and information-seeking behavior.
Recent work showed that EFE minimization can be written as Variational Free Energy (VFE) minimization on a generative model augmented with epistemic priors.
We prove that the VFE of the augmented model can be rewritten as the VFE of the predictive model plus explicit entropy-correction terms, making the EFE contribution transparent.
We then show that proper EFE-based planning requires combining these epistemic corrections with a planning correction that turns marginal inference into policy optimization, yielding a full variational characterization of EFE-based planning.
This clarifies which corrections are needed for cross-entropy planning and for full EFE-based planning.
The same entropy-corrected formulation leads to a detailed message-passing scheme for EFE-based planning together with simpler ablations.
Experiments on three grid-world environments show that full EFE-based planning outperforms ablations that omit either the planning correction or the epistemic corrections.
\end{abstract}


\section{Introduction}
\label{sec:introduction}

Sequential decision-making under uncertainty requires balancing exploitation of current knowledge against exploration to reduce uncertainty.
Classical reinforcement learning and optimal control address this through value functions or policy optimization~\citep{sutton_reinforcement_2018,bertsekas_dynamic_2012}, but typically treat reward maximization and uncertainty reduction as separate objectives.

Planning-as-Inference (PAI) offers an alternative by casting control as probabilistic inference~\citep{attias_planning_2003,toussaint_robot_2009}, connecting control to variational inference and message passing~\citep{levine_reinforcement_2018}.
Standard PAI methods optimize objectives such as expected utility or cross-entropy to preferences, but do not include an explicit epistemic drive to reduce environmental uncertainty.

Active inference (AIF) addresses this by minimizing the Expected Free Energy (EFE), unifying instrumental and epistemic objectives~\citep{friston_active_2015,dacosta_active_2020}.
\citet{nuijten_expected_2026} showed that EFE minimization can be reformulated as Variational Free Energy (VFE) minimization on a model augmented with \emph{epistemic priors}.
This brings active inference into the variational framework, but leaves open a key distinction: obtaining EFE inside a marginal variational objective is not yet the same as planning over policies.
Proper planning additionally requires the planning correction of \citet{lazaro-gredilla_what_2024}.
This paper makes that separation explicit and derives a message-passing scheme for the combined objective.
Our contributions are:
\begin{itemize}
  \item We show that proper EFE-based planning requires combining two entropy corrections: the planning correction of \citet{lazaro-gredilla_what_2024}, which turns the expected-utility variational objective into policy optimization, and the epistemic corrections of \citet{nuijten_expected_2026}, which turn marginal VFE minimization into EFE minimization. Together they yield a full variational characterization of EFE-based planning.
  \item We derive a principled message-passing family for these entropy-corrected objectives. Each added entropy term induces a corresponding channel reparameterization that resolves the circularity of posterior-dependent epistemic priors, and recovers both belief propagation and full active-inference planning within the same derivation.
  \item We validate the framework on three grid-world environments. They differ in where the value of information gathering lies: in some, sensing actions change the observation model itself; in others, a sensing action's only value is \emph{novelty}, the expected information gain about latent parameters. The alternating heuristic of \citet{nuijten_message_2026} finds epistemic actions only in the first case; only the current joint scheme captures novelty.
\end{itemize}
\section{Background}
\label{sec:background}

\subsection{Generative Model for Sequential Decision-Making}
\label{subsec:generative_model}

We consider an agent that maintains a generative model predicting future observations, states, and the consequences of actions.
Following standard conventions~\citep{levine_reinforcement_2018,lazaro-gredilla_what_2024}, we write this as a rollout model:
\begin{align}
  p(\bm{y}, \bm{x}, \bm{u}, \theta) = {} & p(\theta) p(x_0) \prod_{t=1}^T p(y_t | x_t, \theta) \notag                   \\
                                         & \cdot p(x_t | x_{t-1}, u_t, \theta) \, p(u_t)\,, \label{eq:generative_model}
\end{align}
where $\bm{x} = (x_0, \ldots, x_T)$ are latent states, $\bm{y} = (y_1, \ldots, y_T)$ are observations, $\bm{u} = (u_1, \ldots, u_T)$ are actions, and $\theta$ are unknown model parameters.
Here $t = 0$ denotes the current time, and the model predicts a rollout into the future over horizon $T$.
The dynamics $p(x_t | x_{t-1}, u_t, \theta)$ may depend on parameters $\theta$, capturing model uncertainty.
Throughout this paper we work in the discrete regime, so all integrals over $(y_t, x_t, \theta)$ in what follows reduce to finite sums.

To encode goals, we augment the model with \emph{preference priors} $\hat{p}(x_t)$ and $\hat{p}(y_t)$ over desired states and observations~\citep{levine_reinforcement_2018}:
\begin{align}
  \hat{p}(\bm{y}, \bm{x}, \bm{u}, \theta) \propto {} & p(\theta) p(x_0) \prod_{t=1}^T p(y_t | x_t, \theta) \, p(x_t | x_{t-1}, u_t, \theta) \notag \\
                                                     & \cdot p(u_t) \, \hat{p}(x_t) \, \hat{p}(y_t)\,. \label{eq:biased_model}
\end{align}
These preference priors can be understood as proportional to exponentiated rewards: $\hat{p}(x) \propto \exp(R(x))$, connecting planning-as-inference to reward maximization~\citep{todorov_general_2008}.
Together, the rollout model~\eqref{eq:biased_model} with preferences $\hat{p}(x_t)$ and $\hat{p}(y_t)$ defines our planning problem over horizon $T$: find a policy $q(u_t | x_{t-1})$ whose induced predicted trajectory agrees with the preferences.

\subsection{Variational Free Energy}
\label{subsec:vfe}

Given a generative model, variational inference approximates the posterior by minimizing the Variational Free Energy over a family of tractable distributions $q$~\citep{blei_variational_2017}:
\begin{equation} \label{eq:vfe_def}
  F_{\hat{p}}[q] = \KL{q(\bm{y}, \bm{x}, \bm{u}, \theta)}{\hat{p}(\bm{y}, \bm{x}, \bm{u}, \theta)}\,.
\end{equation}
Since all variables are unobserved in the planning setting (they represent future quantities), minimizing $F_{\hat{p}}[q]$ yields beliefs about future trajectories that are consistent with both the dynamics and the preference priors.

\subsection{Factor Graphs and the Bethe Approximation}
\label{subsec:bethe}

The generative model~\eqref{eq:biased_model} factorizes into local terms, which can be represented as a Forney-style factor graph (FFG)~\citep{forney_codes_2001,loeliger_factor_2007}.
In an FFG, nodes represent factors (probability distributions) and edges represent variables; an edge connects to a node when the variable appears in that factor's scope.
We write $\mathcal{E}(a)$ for the set of edges (variables) adjacent to factor node $a$, and $\mathcal{V}(i)$ for the set of factor nodes adjacent to edge $i$. The variables in the scope of factor $a$ are denoted $\bm{s}_a$.

The \emph{Bethe approximation}~\citep{yedidia_constructing_2005} exploits this structure by constraining the variational distribution to respect the factorization induced by the graph.
Each node $a$ maintains a local belief $q_a(\bm{s}_a)$ over its adjacent variables $\bm{s}_a$, and each edge $i$ maintains a singleton belief $q_i(s_i)$.
These beliefs must satisfy local consistency constraints:
\begin{equation}\label{eq:local_consistency}
  \int q_a(\bm{s}_a) \, \dif \bm{s}_{a \setminus i} = q_i(s_i) \quad \text{for all } i \in \mathcal{E}(a)\,.
\end{equation}
Under these constraints, with entropy corrections that prevent double-counting of shared variables, the VFE reduces to the \emph{Bethe Free Energy}:
\begin{align}
  F_{\text{Bethe}}[q] = {} & \sum_{a \in \mathcal{V}} \KL{q_a(\bm{s}_a)}{f_a(\bm{s}_a)} \notag                      \\
                           & + \sum_{i \in \mathcal{E}} (d_i - 1) \, \ent{q_i(s_i)}\,, \label{eq:bethe_free_energy}
\end{align}
where $\mathcal{V}$ is the set of nodes, $\mathcal{E}$ is the set of edges, $f_a$ is the factor at node $a$, and $d_i$ is the degree (number of connected nodes) of edge $i$.
Minimizing the Bethe Free Energy via message passing yields the belief propagation algorithm; on tree-structured graphs, this recovers exact marginals~\citep{pearl_reverend_1982}.
Details are provided in \refappx{appx:bethe}.

\subsection{Epistemic Priors}
\label{subsec:epistemic_priors}

Standard variational inference does not distinguish between variable types: actions, states, observations, and parameters all enter the VFE symmetrically.
\citet{nuijten_expected_2026} clarified the \emph{epistemic priors} $\tilde{p}(u_t)$, $\tilde{p}(x_t)$, and $\tilde{p}(y_t, x_t)$ that encode which variables are controlled, inferred, or observed.
These priors augment the generative model:
\begin{align}
  \tilde{p}(\bm{y}, \bm{x}, \bm{u}, \theta) \propto {} & \hat{p}(\bm{y}, \bm{x}, \bm{u}, \theta) \notag                                                      \\
                                                       & \prod_{t=1}^T \tilde{p}(u_t) \, \tilde{p}(x_t) \, \tilde{p}(y_t, x_t)\,. \label{eq:augmented_model}
\end{align}

Each prior is defined in terms of entropies of conditionals\footnote{We write $\hent{q(y|x)}$ for the \emph{entropy of the conditional} $q(y|x)$, a function of $x$, and $\ent{q(y|x)}$ for the \emph{conditional entropy}, a scalar: $\ent{q(y|x)} = \Ex{q(x)}{\hent{q(y|x)}}$.} of the variational distribution $q$:
\begin{subequations}\label{eq:epistemic_priors}
  \begin{gather}
    \tilde{p}(u_t) \propto \exp\bigl(\hent{q(x_t, x_{t-1} | u_t)} - \hent{q(x_{t-1} | u_t)}\bigr)\,, \label{eq:epistemic_u}  \\
    \tilde{p}(x_t) \propto \exp \bigl(\Ex{q(\theta | x_t)}{-\hent{q(y_t | x_t, \theta)}}\bigr)\,, \label{eq:epistemic_x} \\
    \tilde{p}(y_t, x_t) \propto \exp\bigl(\KL{q(\theta | y_t, x_t)}{q(\theta | x_t)}\bigr)\,. \label{eq:epistemic_xy}
  \end{gather}
\end{subequations}
\citet{nuijten_expected_2026} showed that the VFE of the augmented model $F_{\tilde{p}}[q]$ is an upper bound on the expected EFE.
A notable feature is that the epistemic priors depend on the variational distribution $q$ itself, creating a circular dependency that complicates optimization.
A central contribution of this paper is to make that circularity explicit as entropy corrections in the objective, rather than leaving it implicit in posterior-dependent priors.
\section{Related Work}
\label{sec:related_work}

\paragraph{Planning-as-Inference.}
The PAI framework casts optimal control as inference in graphical models~\citep{attias_planning_2003,toussaint_robot_2009}, connecting control to variational methods and message passing~\citep{levine_reinforcement_2018}.
Closely related formulations include linearly-solvable MDPs~\citep{todorov_linearlysolvable_2006}, path-integral control~\citep{kappen_path_2005}, KL control~\citep{kappen_optimal_2012}, and stochastic optimal control~\citep{rawlik_stochastic_2012}.
A known challenge is \emph{optimistic inference}: conditioning on goals biases posteriors toward trajectories assuming favorable outcomes~\citep{levine_reinforcement_2018}.
This issue was addressed by \citet{lazaro-gredilla_what_2024} with an entropy correction that turns the expected-utility variational objective into a proper control objective by penalizing plans that rely on fortuitous state realizations.

\paragraph{Active inference.}
Active inference minimizes the Expected Free Energy, combining instrumental and epistemic value~\citep{friston_active_2015,dacosta_active_2020,parr_active_2022}.
Existing methods employ specialized procedures: tree search~\citep{friston_sophisticated_2021}, branching~\citep{champion_branching_2022}, or dynamic programming~\citep{paul_efficient_2024}.
The status of the EFE as a variational objective has itself been scrutinized~\citep{millidge_whence_2021}, motivating several works that seek to unify EFE with variational inference.
In a related direction, \citet{palmieri_unifying_2022} combined estimation and control via belief propagation.
Building on the Generalized Free Energy~\citep{parr_generalised_2019}, \citet{koudahl_realising_2023} and \citet{vandelaar_realizing_2024} modified the VFE to include epistemic terms.
Most recently, \citet{nuijten_expected_2026} showed that EFE minimization can be formulated as VFE minimization with epistemic priors, refining the preliminary construction of \citet{devries_expected_2025}, and \citet{nuijten_message_2026} implemented this via message passing with alternating updates between the posterior and epistemic priors.
A separate, complementary line of work~\citep{odonoghue_making_2020,tarbouriech_probabilistic_2023} casts exploration as posterior inference over value functions, targeting uncertainty in the value function rather than in the model parameters~$\theta$.
Our contribution is to connect these lines: the epistemic-prior construction provides the EFE correction to a marginal objective, the L\'azaro-Gredilla construction provides the planning correction, and their combination yields a principled message-passing formulation of EFE-based planning.
\section{Entropy Corrections for EFE-Based Planning}
\label{sec:entropic_inference}

We now show that the epistemic priors from \Cref{subsec:epistemic_priors} and the planning correction of \citet{lazaro-gredilla_what_2024} play different roles.
The epistemic priors identify the corrections that transform marginal VFE minimization into EFE minimization.
The planning correction turns an expected-utility variational objective into a planning objective over policies.
Proper active-inference planning requires both.
More broadly, specifying a planning method is a three-way modeling choice: the generative model, the variable-role assignment among controlled, state, parameter, and observed quantities, and the entropy correction selecting the objective.
The AIF-specific commitment lives entirely in the last.
With no entropy corrections, minimizing the VFE $F_{\hat{p}}[q]$ of the preference-augmented model is simply marginal inference, or in the control setting, KL control~\citep{kappen_optimal_2012}\footnote{Kappen-style tempering $\hat{p}(x) \propto \exp(R(x)/\lambda)$~\citep{kappen_path_2005,kappen_optimal_2012} parameterizes the \emph{generative model} (the preference prior), whereas \Cref{tab:method_comparison} parameterizes the \emph{objective} via entropy corrections; the two axes are orthogonal.}; different objectives arise by adding entropy corrections to this same baseline.
The key question is which corrections are needed for proper EFE-based planning.

\subsection{Cross-Entropy Planning}
\label{subsec:cross_entropy_planning}

Marginal variational inference minimizes a cost over the full $q$, which lets the joint commit to favorable state realizations that the policy alone cannot produce.
\citet{lazaro-gredilla_what_2024} showed that turning this into planning, where the extracted policy $q(u_t|x_{t-1})$ actually attains the cost it appears to minimize, requires an entropy correction that penalizes action uncertainty:
\begin{equation} \label{eq:cross_entropy_correction}
  \sum_{t=1}^T \ent{q(x_{t-1}, u_t)} - \ent{q(x_{t-1})} = \sum_{t=1}^T \ent{q(u_t | x_{t-1})}\,.
\end{equation}
See \refappx{appx:planning_entropy_correction} for the derivation.

Following the control-as-inference framework~\citep{levine_reinforcement_2018}, rewards can be encoded as preference distributions via $\hat{p}(x) \propto \exp(R(x))$.
As shown in \refappx{appx:planning_cross_entropy}, adding the entropy correction \eqref{eq:cross_entropy_correction} to the VFE transforms the objective into minimizing the \emph{cross-entropy} between the state marginals and the preference distribution:
\begin{equation} \label{eq:cross_entropy_objective}
  \min_{q} \sum_{t=1}^T \xent{q(x_t)}{\hat{p}(x_t)} + \text{const}\,,
\end{equation}
where $\xent{q}{\hat{p}} = -\Ex{q}{\log \hat{p}}$ is the cross-entropy.
Since $\xent{q}{\hat{p}} = -\Ex{q}{R(x)} + \text{const}$, minimizing cross-entropy is equivalent to maximizing expected reward.

We call this \textbf{cross-entropy planning}: the agent maximizes expected reward (equivalently, minimizes cross-entropy to preferences) while committing to a policy.

\subsection{EFE as Entropy Corrections}
\label{subsec:active_inference}

The epistemic priors introduced in \Cref{sec:background} augment the generative model with terms that encode variable roles.
The VFE of this augmented model can be expressed as the original VFE plus entropy corrections.
This rewriting is an exact algebraic identity.

\begin{theorem}[Entropy-corrected form of active inference] \label{thm:entropy_decomposition}
  The variational objective of \citet{nuijten_expected_2026} can be written as:
  \begin{multline}\label{eq:aif_correction}
    F_{\tilde{p}}[q] = F_{\hat{p}}[q] + \sum_{t=1}^{T} 2\ent{q(y_t | x_t, \theta)} \\
    - \ent{q(x_t | x_{t-1}, u_t)} - \ent{q(y_t | x_t)}\,.
  \end{multline}
\end{theorem}
\begin{proof}
  See \refappx{appx:entropy_decomposition_proof}.
\end{proof}

Each prior contributes a specific correction: $\tilde{p}(u_t)$ produces $-\ent{q(x_t | x_{t-1}, u_t)}$, $\tilde{p}(x_t)$ produces $+\ent{q(y_t | x_t, \theta)}$, and $\tilde{p}(y_t, x_t)$ contributes a further $+\ent{q(y_t | x_t, \theta)} - \ent{q(y_t | x_t)}$ via the identity $\Ex{q}{\KL{q(\theta|y_t,x_t)}{q(\theta|x_t)}} = \ent{q(y_t|x_t)} - \ent{q(y_t|x_t,\theta)}$, producing the factor of two.
Together these corrections produce EFE minimization inside a marginal variational objective: the $+2\,\ent{q(y_t|x_t,\theta)}$ term pushes beliefs toward state-parameter configurations whose observations are informative, which is what \emph{epistemic} means in AIF.
The signs pull in opposite directions: the negative terms favor spreading belief over reachable states and predicted observations, while the positive term favors concentrating it on informative configurations.
This tension returns as a min-max structure in the message-passing scheme (\Cref{subsec:convergence}).
Grouped differently, the observation-side corrections match standard EFE terminology~\citep{dacosta_active_2020}: one factor of $\ent{q(y_t|x_t,\theta)}$ penalizes \emph{ambiguity}, and the remaining $\ent{q(y_t|x_t)} - \ent{q(y_t|x_t,\theta)}$ enters negatively and rewards \emph{novelty}, the expected information gain about~$\theta$.
Ambiguity depends only on the observation kernel at a given state; novelty depends on the joint posterior over observations, states, and parameters.
This distinction drives the empirical separation in \Cref{sec:experiments}.
By itself, however, \eqref{eq:aif_correction} does not yet yield EFE-based planning, because it lacks the planning correction \eqref{eq:cross_entropy_correction}.
\subsection{EFE-Based Planning}
\label{subsec:efe_based_planning}

The missing step is to combine the marginal-EFE corrections of \Cref{thm:entropy_decomposition} with the planning correction of \Cref{subsec:cross_entropy_planning}.
Adding only the dynamics-side term $-\ent{q(x_t | x_{t-1}, u_t)}$ to cross-entropy planning yields \textbf{risk-minimizing planning}, which minimizes what \citet{friston_active_2015} call \emph{risk}; it is a useful intermediate ablation, but not yet full EFE-based planning because it omits the observation-side epistemic corrections.

\refappx{appx:efe_planning_inference} proves that the resulting EFE-based planning objective is
\begin{align}
  \min_q\; & F_{\hat{p}}[q]
  + \sum_{t=1}^T \ent{q(u_t | x_{t-1})}
  + \sum_{t=1}^T \Bigl(
  2\ent{q(y_t | x_t, \theta)} \notag \\
           & \qquad\quad
  - \ent{q(x_t | x_{t-1}, u_t)}
  - \ent{q(y_t | x_t)}
  \Bigr)\,. \label{eq:combined_main_text}
\end{align}
The first sum is the planning correction; the second sum is the EFE correction.
Only their combination yields proper \textbf{EFE-based planning}.

\subsection{Comparison of Objectives}
\label{subsec:landscape}

\begin{table*}[t]
  \centering
  \caption{Entropy corrections needed to move from baseline variational inference to proper EFE-based planning.}
  \label{tab:method_comparison}
  \begin{tabular}{l | l}
    \toprule
    \textbf{Objective}                                       & \textbf{Added entropy correction}                                                            \\
    \midrule
    Baseline VFE / Marginal inference                        & $0$                                                                                          \\
    Cross-entropy planning~\citep{lazaro-gredilla_what_2024} & $+\sum_t \ent{q(u_t | x_{t-1})}$                                                             \\
    Risk-minimizing planning                                 & CE $- \sum_t \ent{q(x_t | x_{t-1}, u_t)}$                                                    \\
    EFE-based planning                                       & CE $+ \sum_t 2\ent{q(y_t | x_t, \theta)} - \ent{q(x_t | x_{t-1}, u_t)} - \ent{q(y_t | x_t)}$ \\
    \bottomrule
  \end{tabular}
\end{table*}
\Cref{tab:method_comparison} summarizes the progression: the planning correction changes \emph{how} control is posed (cross-entropy planning), the EFE correction changes \emph{what} objective is optimized (marginal EFE), and only their combination yields proper EFE-based planning, the objective implemented in our experiments.
The channel reparameterizations required for message passing follow directly from these correction terms (\Cref{subsec:channel_reparam}); we turn to that next.
\section{Message Passing for EFE-Based Planning}
\label{sec:message_passing}

The full EFE-based planning objective \eqref{eq:combined_main_text} adds four conditional entropy terms to the VFE: one for the policy, one for the dynamics, and two for the observation model.
In the Bethe framework a conditional distribution is a ratio of region beliefs, so these terms are not functions of any single coordinate of the optimization problem.
We resolve this by introducing auxiliary conditional distributions (\emph{channels}) that promote the corrected conditionals to free variational parameters, yielding a message-passing family that generalizes standard belief propagation.
\subsection{Channel Reparameterization}
\label{subsec:channel_reparam}

\begin{figure}[t]
    \centering
    \includegraphics[width=\columnwidth]{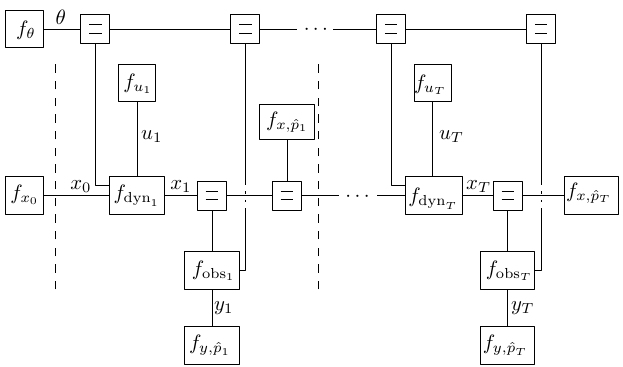}
    \caption{Forney factor graph for the generative model~\eqref{eq:generative_model}. Square nodes are factors; edges are variables. The time slice between the dashed lines is repeated for $T$ timesteps.}
    \label{fig:factor_graph}
\end{figure}

The FFG representation (\Cref{subsec:bethe}) makes the locality of channel reparameterization explicit: each correction acts on a single factor node, while the remainder of the graph is unchanged from standard sum-product (\Cref{fig:factor_graph}).
The key identity is the variational characterization of conditional entropy (Gibbs' inequality):
\begin{equation}\label{eq:gibbs_identity}
    \ent{q(y | x)} = \min_{r} \Ex{q(y, x)}{-\log r(y | x)}\,,
\end{equation}
with equality when $r(y|x) = q(y|x)$, so the substitution is exact rather than a bound.
We introduce four normalized conditional distributions as \emph{channels}: $r_{u|x,t}(u_t | x_{t-1})$, $r_{x|xu,t}(x_t | x_{t-1}, u_t)$, $r_{y|x\theta,t}(y_t | x_t, \theta)$, and $r_{y|x,t}(y_t | x_t)$, as free variational parameters for each time step~$t$ (see \refappx{appx:combined_detailed_derivation} for formal definitions).
Substituting \eqref{eq:gibbs_identity} into the corrections of \eqref{eq:combined_main_text} yields a well-posed optimization in which each channel enters the factor it corrects according to the sign of its entropy term: the positive corrections place $r_{u|x,t}$ and $r_{y|x\theta,t}$ (squared) in numerators, while the negative corrections place $r_{x|xu,t}$ and $r_{y|x,t}$ in denominators.
This yields the \emph{kernels}:
\begin{subequations}\label{eq:modified_kernels}
    \begin{equation}
        \tilde{f}_{\mathrm{obs}_t}(y_t, x_t, \theta) = \frac{p(y_t | x_t, \theta)\, r_{y|x\theta,t}^2(y_t | x_t, \theta)}{r_{y|x,t}(y_t | x_t)}\,, \label{eq:modified_obs}
    \end{equation}
    \begin{equation}
        \tilde{f}_{\mathrm{dyn}_t}(x_t, x_{t-1}, \theta, u_t) = \frac{p(x_t | x_{t-1}, \theta, u_t)\, r_{u|x,t}(u_t | x_{t-1})}{r_{x|xu,t}(x_t | x_{t-1}, u_t)}. \label{eq:modified_dyn}
    \end{equation}
\end{subequations}
With these substitutions, the EFE-based planning objective becomes a standard Bethe free energy over the modified factor graph, jointly optimized over beliefs and channels.
The kernels~\eqref{eq:modified_kernels} replace the original factor functions in the message-passing equations, making the procedure iterative: the channel beliefs $r$ depend on variational beliefs $q$ and vice versa. The proof for $T=1$ is given in \refappx{appx:combined_detailed_derivation}. The full scheme comes from the additivity of the Lagrangian and the entropic corrections, see \refappx{appx:combined_detailed_generic} for details.

\subsection{Message-Passing Equations}
\label{subsec:mp_equations}

\begin{figure}[t]
    \centering
    \includegraphics[width=0.85\columnwidth]{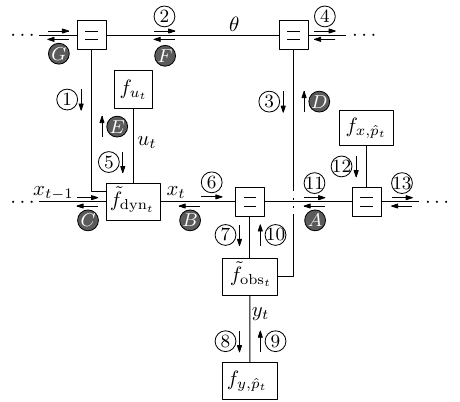}
    \caption{Factor graph for a single time slice of the generative model~\eqref{eq:generative_model}. The observation factor $f_{\mathrm{obs}_t}$ and dynamics factor $f_{\mathrm{dyn}_t}$ receive kernels $\tilde{f}_{\mathrm{obs}_t}$ and $\tilde{f}_{\mathrm{dyn}_t}$ from \eqref{eq:modified_kernels} under the EFE-based planning objective. Numbered messages are computed in a forward pass over all time slices, lettered messages in a subsequent backward pass.}
    \label{fig:factor_graph_labeled_messages}
\end{figure}

Since the modified objective has a Bethe form, the stationarity conditions yield sum-product-style message updates. The only difference from standard belief propagation is that each factor uses its kernel \eqref{eq:modified_kernels} in place of the original.

Each factor~$a$ sends to a neighboring factor $b$ the integral of its kernel over incoming messages on adjacent edges:
\begin{equation}\label{eq:factor_to_var}
    \mu_{jb}(s_j) \propto \!\int\! \tilde{f}_a(\bm{s}_a)\!\! \prod_{i \in \mathcal{E}(a) \setminus j}\!\! \mu_{ia}(s_i) \, \dif \bm{s}_{a \setminus j}\,.
\end{equation}
For unmodified factors (priors, data likelihoods), $\tilde{f}_a = f_a$.
\paragraph{Singleton beliefs.}
Singleton beliefs are computed by normalizing the product of colliding messages on an edge,
\begin{equation}\label{eq:singleton_beliefs}
    q^*(s_i) \propto \mu_{ia}(s_i)\mu_{ib}(s_i)\,,
\end{equation}
with $\{a,b\} = \mathcal{V}(i)$ the nodes adjacent to edge $i$. The full forward-backward schedule is shown in \Cref{fig:factor_graph_labeled_messages}.

\paragraph{Region beliefs.}
Region beliefs are computed by multiplying the factor function with all inbound messages and normalizing,
\begin{equation} \label{eq:region-belief-update}
    q^*(\bm{s}_a) \propto \tilde{f}_a(\bm{s}_a) \prod_{i \in \mathcal{E}(a)} \mu_{ia}(s_i)\,.
\end{equation}

\paragraph{Channel updates.}
At the fixed point, each channel recovers the true conditional under its factor belief:
\begin{subequations}\label{eq:channel_updates}
    \begin{align}
        r_{u|x,t}^*(u_t | x_{t-1})           & = q_{u|x,t}^*(u_t | x_{t-1})\,, \label{eq:channel_policy}        \\
        r_{y|x\theta,t}^*(y_t | x_t, \theta) & = q_{y|x\theta,t}^*(y_t | x_t, \theta)\,, \label{eq:channel_obs} \\
        r_{y|x,t}^*(y_t | x_t)               & = q_{y|x,t}^*(y_t | x_t)\,, \label{eq:channel_obs_marginal}      \\
        r_{x|xu,t}^*(x_t | x_{t-1}, u_t)     & = q_{x|xu,t}^*(x_t | x_{t-1}, u_t)\,, \label{eq:channel_dyn}
    \end{align}
\end{subequations}
where the beliefs $q$ are conditionals derived from the respective region beliefs around factors $f_{\mathrm{obs}_t}$ and $f_{\mathrm{dyn}_t}$ at time~$t$. The marginal observation channel $r_{y|x,t}^*$ is obtained by marginalizing $\theta$ from the observation factor belief (see \refappx{appx:combined_detailed_derivation}).

\paragraph{Learning vs. planning.}
Learning and planning are kept separate. The parameters $\theta$ are updated by Bayesian filtering on real observations as they arrive, whereas during planning $\theta$ is held fixed at the current belief $q(\theta)$: the backward messages into $\theta$ from the (simulated) observation and dynamics factors are not sent. This is what lets the global parameter information gain decompose into the per-step novelty factors $\tilde{p}(y_t, x_t)$ of \eqref{eq:epistemic_xy}. Each step's novelty is then scored against a common $\theta$ baseline, so it is additive across the horizon. The omitted messages are well-defined; holding $\theta$ fixed during planning keeps per-step novelty well-posed.

\begin{algorithm}[t]
    \caption{EFE-Based Planning Message Passing}\label{alg:aif_mp}
    \begin{algorithmic}[1]
        \Require Generative model factors $\{f_{\mathrm{obs}_t}, f_{\mathrm{dyn}_t}\}_{t=1}^T$, priors $p(\theta), p(x_0), p(u_t)$, goal priors $\hat{p}_x(x_t), \hat{p}_y(y_t)$
        \Ensure Action beliefs $\{q_{u_t}^*(u_t)\}_{t=1}^T$
        \State Initialize all messages $\mu \gets 1$, channels $r_{u|x}, r_{y|x\theta}, r_{y|x}, r_{x|xu} \gets$ uniform
        \Repeat
        \For{$t = 1, \ldots, T$}
        \State Compute sum-product messages~\eqref{eq:factor_to_var}
        \State Update region beliefs~\eqref{eq:region-belief-update}
        \State Update channels~\eqref{eq:channel_updates}, \eqref{eq:damping}
        \State Update kernels~\eqref{eq:modified_kernels}
        \EndFor
        \State Update singleton beliefs~\eqref{eq:singleton_beliefs}
        \Until{convergence}
        \State \Return $\{q_{u_t}^*(u_t)\}_{t=1}^T$
    \end{algorithmic}
\end{algorithm}
The same construction yields a family of algorithms: Value Belief Propagation (VBP)~\citep{lazaro-gredilla_what_2024} uses only the policy reparameterization, risk-minimizing planning (\Cref{tab:method_comparison}) additionally uses the dynamics-side reparameterization, and full EFE-based planning further adds the observation-side reparameterizations. The corresponding VBP derivation is given in \refappx{appx:vbp_derivation}.

\subsection{Convergence}
\label{subsec:convergence}
The kernels~\eqref{eq:modified_kernels} contain opposing channel corrections (\Cref{sec:entropic_inference}): $r_{u|x}$ and $r_{y|x\theta}$ appear in numerators, while $r_{x|xu}$ and $r_{y|x}$ appear in denominators, inducing a min-max structure in the joint optimization.
Because each channel reparameterization is a local rewrite of a single kernel, the per-update cost matches standard loopy belief propagation up to the channel updates.
Standard BP convergence guarantees, however, do not transfer to this min-max setting, and we apply geometric damping to channels for each update $n$:
\begin{equation}\label{eq:damping}
    r_c^{n} \propto (r_c^{n-1})^{1 - \lambda}\, (r_c^*)^{\lambda}\,,
\end{equation}
for each channel $c \in \{u|x, x|xu, y|x\theta, y|x\}$. Here $\lambda \in [0, 1]$ is the damping parameter and $r_c^*$ denotes the newly computed channel from \eqref{eq:channel_updates}. We select $\lambda$ per method and environment from a convergence sweep; at the selected value the channel-based methods reach a stationary VFE plateau, typically within $150$ iterations (see \refappx{appx:convergence}).
\section{Experiments}
\label{sec:experiments}
We design experiments to test the behavioral effect of progressively adding the entropy corrections in \Cref{tab:method_comparison}.
Full experimental details are deferred to \refappx{appx:experiments}\footnote{Code available at \url{https://github.com/biaslab/UAI-MP-AIF-JAX}}.
All environments use discrete state spaces with exact factor evaluations, isolating the effect of the entropy corrections and channel-augmented schemes from errors introduced by approximate message computation.

\subsection{Setup}
\label{subsec:setup}

\paragraph{Environments.}
We adapt three classic grid-world environments into epistemic planning benchmarks by treating the environment layout as an unknown parameter~$\theta$ in the generative model~\eqref{eq:generative_model}.
All environments support cardinal movement but differ in how epistemic value arises, which determines which observation-side corrections a method must capture.
In Frozen Lake, information gathering changes the observation model itself, so its value is visible as reduced ambiguity.
In canonical RockSample and Wumpus World, sensing actions affect only a single observation, so their value is pure novelty: it lies entirely in what the resulting reading reveals about~$\theta$.

\emph{Frozen Lake}~\citep{brockman_openai_2016,towers_gymnasium_2025}:
The agent observes binary ``hole/safe'' sensors for every cell on the grid, with noise that increases with distance from the agent.
A SCAN action switches the agent into a persistent scan mode in which all observations become near-deterministic, at the cost of one time step and a lower prior preference.

\emph{RockSample~$(4,3)$}~\citep{smith_heuristic_2004}:
The canonical RockSample benchmark on a $4{\times}4$ grid with $3$ rocks at known positions and unknown quality (good or bad), defining~$\theta$ ($8$ configurations).
The agent has one CHECK action per rock, returning a noisy quality reading whose accuracy degrades with distance to that rock, and receives no rock information otherwise.
It can SAMPLE a rock for a reward or penalty depending on quality, or EXIT for a fixed reward.

\emph{Wumpus World}~\citep{russell_artificial_1995}:
Pit, wumpus, and gold positions define~$\theta$ ($25$ configurations).
The classic dynamics are simplified to isolate the epistemic challenge: the agent has no orientation or inventory and navigates by cardinal movement.
The agent observes noisy breeze, stench, and glitter adjacency signals and has uncertain position.
A SCAN action sharpens the adjacency signals for a single step, after which the noise reverts; position channels are unaffected.
Even a precise reading only narrows~$\theta$: a breeze indicates a nearby pit but not which neighbor, so the agent must triangulate across multiple positions.

\paragraph{Methods.}
We compare five methods. The first four correspond to message-passing implementations of the entropy-corrected objectives in \Cref{subsec:landscape}, with channel configurations as specified in \Cref{alg:aif_mp}:
\begin{enumerate}
  \item \textbf{BP}: standard belief propagation, no entropy correction.
  \item \textbf{VBP}: cross-entropy planning, implemented as the principled channelized scheme from \refappx{appx:vbp_derivation}.
  \item \textbf{RM-MP}: risk-minimizing planning (\Cref{tab:method_comparison}), using the planning channel together with the dynamics channel; reduces to VBP under deterministic dynamics.
  \item \textbf{AIF-MP}: full EFE-based planning, using the planning, dynamics, and observation channels (\Cref{alg:aif_mp}).
  \item \textbf{Nuijten-MP} \citep{nuijten_message_2026}: an alternating heuristic that treats the epistemic priors as literal prior factors, recomputed from the current posterior between belief-propagation sweeps, and does not incorporate the novelty prior~\eqref{eq:epistemic_xy}.
\end{enumerate}
All methods except BP include the planning correction, so the experiments ablate the EFE-side corrections on top of a fixed planning baseline rather than the planning correction itself.

\subsection{Results and Discussion}
\label{subsec:results}

\Cref{tab:results} reports performance for all methods across three environments.
The results show where each correction matters and where accounting for novelty becomes necessary.

\begin{table*}[t]
  \centering
  \caption{Performance across three environments with 95\% confidence intervals, averaged over 1000 episodes. Best per metric (non-overlapping CIs) in bold.}
  \label{tab:results}
  \setlength{\tabcolsep}{4pt}
  \begin{tabular}{l c c c c}
    \toprule
                    & \textbf{Frozen Lake}      & \multicolumn{2}{c}{\textbf{RockSample}} & \textbf{Wumpus World}                             \\
    \cmidrule(lr){3-4}
    \textbf{Method} & Success (\%)              & Avg.\ reward                            & Retrieval (\%)        & Success (\%)              \\
    \midrule
    BP              & $51.9\;[48.8, 55.0]$      & $1.00\;[1.00, 1.00]$                    & $0.0$                 & $1.2\;[0.5, 1.9]$         \\
    VBP             & $54.5\;[51.4, 57.6]$      & $1.00\;[1.00, 1.00]$                    & $0.0$                 & $5.5\;[4.1, 6.9]$         \\
    RM-MP           & $50.0\;[46.9, 53.1]$      & $1.00\;[1.00, 1.00]$                    & $0.0$                 & $24.0\;[21.4, 26.6]$      \\
    Nuijten-MP      & $\bm{95.6}\;[94.3, 96.9]$ & $1.00\;[1.00, 1.00]$                    & $0.0$                 & $5.0\;[3.6, 6.4]$         \\
    AIF-MP          & $\bm{95.9}\;[94.7, 97.1]$ & $\bm{4.01}\;[3.90, 4.12]$               & $\bm{98.7}$           & $\bm{40.7}\;[37.7, 43.7]$ \\
    \bottomrule
  \end{tabular}
\end{table*}

\paragraph{Epistemic actions that change the observation model (Frozen Lake).}
Both active inference methods dominate (${\sim}96\%$ success, overlapping confidence intervals), substantially outperforming all baselines.
Both learn to SCAN.
Because scan mode persistently changes the observation kernel, reaching scan-mode states already pays off through reduced ambiguity, so the alternating heuristic finds the epistemic action as readily as the joint scheme.
RM-MP performs comparably to BP and VBP (overlapping confidence intervals): the dynamics correction is neither beneficial nor harmful in this regime.

\paragraph{Novelty-driven sensing (RockSample and Wumpus World).}
In canonical RockSample, no action changes the observation model: a CHECK buys a single noisy reading whose only value is the information it carries about rock quality.
AIF-MP is the only method that checks and samples, retrieving $98.7\%$ of good rocks for an average reward of $4.01$.
Every baseline, including Nuijten-MP, walks straight to the exit (reward $1.00$, zero retrieval): under the uniform quality prior, sampling an unchecked rock has negative expected reward, so without the novelty term CHECK has no value and EXIT is optimal.
Wumpus World repeats this pattern under local readings: even a precise one-step scan does not reveal the global layout, since multiple hazard configurations produce the same breeze and stench patterns, so a scan only narrows~$\theta$.
AIF-MP again clearly leads ($40.7\%$), while Nuijten-MP performs at the level of VBP ($5.0\%$ vs.\ $5.5\%$, overlapping confidence intervals): without an exploitable change in the observation model, the alternating heuristic reduces to its planning-only core.
RM-MP is the strongest baseline ($24.0\%$): the slip-perturbed dynamics activate its dynamics channel, but without the observation-side corrections it cannot value what a scan reveals about~$\theta$.
The remaining gap is an objective mismatch, not a scheduling artifact: \citet{nuijten_message_2026} treat the epistemic priors as literal prior factors, recomputed \emph{outside} the variational objective between belief-propagation sweeps, and do not incorporate the novelty prior~\eqref{eq:epistemic_xy}, whereas AIF-MP treats all four channels as variational parameters of a single joint objective with closed-form stationary conditions~\eqref{eq:channel_updates}, through which the expected information gain about~$\theta$ propagates into the plan.
Representative trajectories are shown in \refappx{appx:wumpus_trajectories}.
\paragraph{Synthesis.}
The alternating heuristic captures half of the observation-side story: it responds to ambiguity, the precision of the observation kernel at reachable states, but not to novelty, the information observations carry about~$\theta$.
Ambiguity stands in for novelty when an epistemic action persistently changes the observation model (Frozen Lake); when sensing actions affect only a single observation, the heuristic collapses to planning-only performance while AIF-MP does not degrade.
Across all environments, the planning correction yields modest gains and the dynamics correction helps only under stochastic dynamics (Wumpus World); most of the gap to AIF-MP is explained by the observation-level corrections.
Accounting for novelty requires the joint variational treatment; treating the epistemic priors as literal priors and leaving out the novelty prior does not suffice.
\section{Conclusion}
\label{sec:conclusion}

This paper clarifies the variational structure of active inference planning.
\Cref{thm:entropy_decomposition} shows that the epistemic-prior construction of \citet{nuijten_expected_2026} admits an explicit entropy-corrected reformulation: relative to baseline VFE minimization, it adds a specific set of entropy corrections that yields marginal EFE minimization, making explicit which terms contribute the epistemic part of the objective.
Proper EFE-based planning additionally requires the planning correction of \citet{lazaro-gredilla_what_2024}, and the combined objective leads directly to a message-passing construction via channel reparameterization.
That construction recovers a family of algorithms, including VBP, risk-minimizing planning, and full EFE-based planning (\Cref{alg:aif_mp}).

Empirically, the planning and dynamics corrections account for only modest gains, while the observation-side corrections separate along the ambiguity/novelty split: only the joint channelized scheme captures novelty, the expected information gain about model parameters, and sustains performance when sensing actions affect only a single observation.

\paragraph{Limitations and future work.}
The opposing signs of the entropy corrections induce a min-max structure in the joint optimization over beliefs and channels, which in practice requires damped channel updates for stable convergence (\Cref{subsec:convergence}).
Standard belief propagation convergence guarantees do not transfer to this setting, and developing convergence theory for the channel-augmented scheme is an open problem; the damping parameter $\lambda$ also currently requires manual per-environment adjustment.
We restrict to discrete state spaces where exact factor evaluations are available; understanding how the channel reparameterization interacts with further factorization constraints on the variational posterior (e.g., mean-field or structured approximations) is an important direction.


\begin{contributions}
  W.W.L.~Nuijten and M.~Lukashchuk contributed equally to this work.
  W.W.L.~Nuijten developed the entropy decomposition framework.
  M.~Lukashchuk derived the message-passing scheme.
  Both authors contributed to writing and experiments.
  T.~van~de~Laar contributed to the conceptualization of the method and supervision.
  B.~de~Vries has a supervisory and editorial role.
\end{contributions}

\begin{acknowledgements}
  This publication is part of the project ROBUST: Trustworthy AI-based Systems for Sustainable Growth with project number KICH3.LTP.20.006, which is (partly) financed by the Dutch Research Council (NWO), GN Hearing, and the Dutch Ministry of Economic Affairs and Climate Policy (EZK) under the program LTP KIC 2020-2023.
\end{acknowledgements}

\bibliography{references}

@inproceedings{attias_planning_2003,
  title = {Planning by Probabilistic Inference},
  booktitle = {International Workshop on Artificial Intelligence and Statistics},
  author = {Attias, Hagai},
  year = 2003,
  pages = {9--16},
  publisher = {PMLR},
  url = {https://proceedings.mlr.press/r4/attias03a.html},
  urldate = {2025-05-04}
}

@book{bertsekas_dynamic_2012,
  title = {Dynamic Programming and Optimal Control: {{Volume I}}},
  shorttitle = {Dynamic Programming and Optimal Control},
  author = {Bertsekas, Dimitri},
  year = 2012,
  volume = {4},
  publisher = {Athena scientific},
  url = {https://books.google.com/books?hl=en&lr=&id=qVBEEAAAQBAJ&oi=fnd&pg=PR1&dq=Dynamic+Programming+and+Optimal+Control&ots=x0bAav0O5n&sig=s3UxthkdnzR2UpqCUsUsQ7zKgLc},
  urldate = {2025-05-04}
}

@article{blei_variational_2017,
  title = {Variational {{Inference}}: {{A Review}} for {{Statisticians}}},
  shorttitle = {Variational {{Inference}}},
  author = {Blei, David M. and Kucukelbir, Alp and McAuliffe, Jon D.},
  year = 2017,
  month = apr,
  journal = {Journal of the American Statistical Association},
  volume = {112},
  number = {518},
  pages = {859--877},
  publisher = {Taylor \& Francis},
  issn = {0162-1459},
  doi = {10.1080/01621459.2017.1285773},
  url = {https://doi.org/10.1080/01621459.2017.1285773},
  urldate = {2023-07-25}
}

@misc{bradbury_jax_2018,
  title = {{{JAX}}: Composable Transformations of {{Python}}+{{NumPy}} Programs},
  author = {Bradbury, James and Frostig, Roy and Hawkins, Peter and Johnson, Matthew James and Leary, Chris and Maclaurin, Dougal and Necula, George and Paszke, Adam and VanderPlas, Jake and {Wanderman-Milne}, Skye and Zhang, Qiao},
  year = 2018,
  url = {http://github.com/jax-ml/jax}
}

@misc{brockman_openai_2016,
  title = {{{OpenAI Gym}}},
  author = {Brockman, Greg and Cheung, Vicki and Pettersson, Ludwig and Schneider, Jonas and Schulman, John and Tang, Jie and Zaremba, Wojciech},
  year = 2016,
  month = jun,
  number = {arXiv:1606.01540},
  eprint = {1606.01540},
  primaryclass = {cs},
  publisher = {arXiv},
  doi = {10.48550/arXiv.1606.01540},
  url = {http://arxiv.org/abs/1606.01540},
  urldate = {2024-04-08},
  archiveprefix = {arXiv}
}

@article{champion_branching_2022,
  title = {Branching {{Time Active Inference}}: {{The}} Theory and Its Generality},
  shorttitle = {Branching {{Time Active Inference}}},
  author = {Champion, Th{\'e}ophile and Da Costa, Lancelot and Bowman, Howard and Grze{\'s}, Marek},
  year = 2022,
  month = jul,
  journal = {Neural Networks},
  volume = {151},
  pages = {295--316},
  issn = {0893-6080},
  doi = {10.1016/j.neunet.2022.03.036},
  url = {https://www.sciencedirect.com/science/article/pii/S0893608022001149},
  urldate = {2025-05-06}
}

@article{dacosta_active_2020,
  title = {Active Inference on Discrete State-Spaces: {{A}} Synthesis},
  shorttitle = {Active Inference on Discrete State-Spaces},
  author = {Da Costa, Lancelot and Parr, Thomas and Sajid, Noor and Veselic, Sebastijan and Neacsu, Victorita and Friston, Karl},
  year = 2020,
  month = dec,
  journal = {Journal of Mathematical Psychology},
  volume = {99},
  pages = {102447},
  issn = {0022-2496},
  doi = {10.1016/j.jmp.2020.102447},
  url = {https://www.sciencedirect.com/science/article/pii/S0022249620300857},
  urldate = {2023-08-07},
  langid = {english}
}

@misc{devries_expected_2025,
  title = {Expected {{Free Energy-based Planning}} as {{Variational Inference}}},
  author = {De Vries, Bert and Nuijten, Wouter and {van de Laar}, Thijs and Kouw, Wouter and Adamiat, Sepideh and Nisslbeck, Tim and Lukashchuk, Mykola and Nguyen, Hoang Minh Huu and Araya, Marco Hidalgo and Tresor, Raphael and Jenneskens, Thijs and Nikoloska, Ivana and Subramanian, Raaja Ganapathy and van Erp, Bart and Bagaev, Dmitry and Podusenko, Albert},
  year = 2025,
  month = apr,
  number = {arXiv:2504.14898},
  eprint = {2504.14898},
  primaryclass = {stat},
  publisher = {arXiv},
  doi = {10.48550/arXiv.2504.14898},
  url = {http://arxiv.org/abs/2504.14898},
  urldate = {2025-05-01},
  archiveprefix = {arXiv}
}

@article{forney_codes_2001,
  title = {Codes on Graphs: {{Normal}} Realizations},
  shorttitle = {Codes on Graphs},
  author = {Forney, G. David},
  year = 2001,
  journal = {IEEE Transactions on Information Theory},
  volume = {47},
  number = {2},
  pages = {520--548},
  publisher = {IEEE},
  url = {https://ieeexplore.ieee.org/abstract/document/910573/},
  urldate = {2025-05-05}
}

@article{friston_active_2015,
  title = {Active Inference and Epistemic Value},
  author = {Friston, Karl and Rigoli, Francesco and Ognibene, Dimitri and Mathys, Christoph and Fitzgerald, Thomas and Pezzulo, Giovanni},
  year = 2015,
  month = oct,
  journal = {Cognitive Neuroscience},
  volume = {6},
  number = {4},
  pages = {187--214},
  issn = {1758-8928, 1758-8936},
  doi = {10.1080/17588928.2015.1020053},
  url = {http://www.tandfonline.com/doi/full/10.1080/17588928.2015.1020053},
  urldate = {2025-05-04},
  langid = {english}
}

@article{friston_sophisticated_2021,
  title = {Sophisticated {{Inference}}},
  author = {Friston, Karl and Da Costa, Lancelot and Hafner, Danijar and Hesp, Casper and Parr, Thomas},
  year = 2021,
  month = mar,
  journal = {Neural Computation},
  volume = {33},
  number = {3},
  pages = {713--763},
  issn = {0899-7667},
  doi = {10.1162/neco_a_01351},
  url = {https://doi.org/10.1162/neco_a_01351},
  urldate = {2021-12-22}
}

@article{heskes_convexity_2006,
  title = {Convexity {{Arguments}} for {{Efficient Minimization}} of the {{Bethe}} and {{Kikuchi Free Energies}}},
  author = {Heskes, T.},
  year = 2006,
  month = jun,
  journal = {Journal of Artificial Intelligence Research},
  volume = {26},
  pages = {153--190},
  issn = {1076-9757},
  doi = {10.1613/jair.1933},
  url = {https://www.jair.org/index.php/jair/article/view/10456},
  urldate = {2023-05-01},
  copyright = {Copyright (c)},
  langid = {english}
}

@article{kappen_optimal_2012,
  title = {Optimal Control as a Graphical Model Inference Problem},
  author = {Kappen, B. and Gomez, V. and Opper, M.},
  year = 2012,
  month = may,
  journal = {Machine Learning},
  volume = {87},
  number = {2},
  eprint = {0901.0633},
  primaryclass = {cs, math},
  pages = {159--182},
  issn = {0885-6125, 1573-0565},
  doi = {10.1007/s10994-012-5278-7},
  url = {http://arxiv.org/abs/0901.0633},
  urldate = {2024-07-31},
  archiveprefix = {arXiv}
}

@article{kappen_path_2005,
  title = {Path Integrals and Symmetry Breaking for Optimal Control Theory},
  author = {Kappen, H J},
  year = 2005,
  month = nov,
  journal = {Journal of Statistical Mechanics: Theory and Experiment},
  volume = {2005},
  number = {11},
  pages = {P11011},
  issn = {1742-5468},
  doi = {10.1088/1742-5468/2005/11/P11011},
  url = {https://doi.org/10.1088/1742-5468/2005/11/P11011},
  urldate = {2026-05-20},
  langid = {english}
}

@misc{koudahl_realising_2023,
  title = {Realising {{Synthetic Active Inference Agents}}, {{Part I}}: {{Epistemic Objectives}} and {{Graphical Specification Language}}},
  shorttitle = {Realising {{Synthetic Active Inference Agents}}, {{Part I}}},
  author = {Koudahl, Magnus and {van de Laar}, Thijs and {de Vries}, Bert},
  year = 2023,
  month = jun,
  number = {arXiv:2306.08014},
  eprint = {2306.08014},
  primaryclass = {cs},
  publisher = {arXiv},
  doi = {10.48550/arXiv.2306.08014},
  url = {http://arxiv.org/abs/2306.08014},
  urldate = {2023-09-19},
  archiveprefix = {arXiv}
}

@inproceedings{lazaro-gredilla_what_2024,
  title = {What Type of Inference Is Planning?},
  booktitle = {Advances in {{Neural Information Processing Systems}}},
  author = {{L{\'a}zaro-Gredilla}, Miguel and Ku, Li Yang and Murphy, Kevin P. and George, Dileep},
  editor = {Globerson, A. and Mackey, L. and Belgrave, D. and Fan, A. and Paquet, U. and Tomczak, J. and Zhang, C.},
  year = 2024,
  volume = {37},
  pages = {116705--116742},
  publisher = {Curran Associates, Inc.},
  doi = {10.52202/079017-3705},
  url = {https://proceedings.neurips.cc/paper_files/paper/2024/file/d39e3ae9a11b79691709a7a6e06a63d9-Paper-Conference.pdf}
}

@misc{levine_reinforcement_2018,
  title = {Reinforcement {{Learning}} and {{Control}} as {{Probabilistic Inference}}: {{Tutorial}} and {{Review}}},
  shorttitle = {Reinforcement {{Learning}} and {{Control}} as {{Probabilistic Inference}}},
  author = {Levine, Sergey},
  year = 2018,
  month = may,
  number = {arXiv:1805.00909},
  eprint = {1805.00909},
  primaryclass = {cs},
  publisher = {arXiv},
  doi = {10.48550/arXiv.1805.00909},
  url = {http://arxiv.org/abs/1805.00909},
  urldate = {2025-04-24},
  archiveprefix = {arXiv}
}

@article{loeliger_factor_2007,
  title = {The {{Factor Graph Approach}} to {{Model-Based Signal Processing}}},
  author = {Loeliger, Hans-Andrea and Dauwels, Justin and Hu, Junli and Korl, Sascha and Ping, Li and Kschischang, Frank R.},
  year = 2007,
  month = jun,
  journal = {Proceedings of the IEEE},
  volume = {95},
  number = {6},
  pages = {1295--1322},
  issn = {0018-9219},
  doi = {10.1109/JPROC.2007.896497},
  urldate = {2014-04-10}
}

@article{millidge_whence_2021,
  title = {Whence the {{Expected Free Energy}}?},
  author = {Millidge, Beren and Tschantz, Alexander and Buckley, Christopher L.},
  year = 2021,
  month = feb,
  journal = {Neural Computation},
  volume = {33},
  number = {2},
  pages = {447--482},
  issn = {0899-7667},
  doi = {10.1162/neco_a_01354},
  url = {https://ieeexplore.ieee.org/abstract/document/9346140},
  urldate = {2026-06-15}
}

@article{nuijten_expected_2026,
  title = {Expected Free Energy-Based Planning as Variational Inference},
  author = {Nuijten, Wouter W. L. and {van de Laar}, Thijs and {de Vries}, Bert},
  year = 2026,
  journal = {Transactions on Machine Learning Research},
  issn = {2835-8856},
  url = {https://openreview.net/forum?id=Kzm8I1oS1s}
}

@inproceedings{nuijten_message_2026,
  title = {A {{Message Passing Realization}} of~{{Expected Free Energy Minimization}}},
  booktitle = {Active {{Inference}}},
  author = {Nuijten, Wouter W. L. and Lukashchuk, Mykola and {van de Laar}, Thijs and {de Vries}, Bert},
  editor = {Albarracin, Mahault and Benrimoh, David and Buckley, Christopher L. and Lanillos, Pablo and Pitliya, Riddhi J. and Shimazaki, Hideaki and Stoianov, Ivilin Peev and Verbelen, Tim and Wisse, Martijn},
  year = 2026,
  pages = {75--98},
  publisher = {Springer Nature Switzerland},
  address = {Cham},
  doi = {10.1007/978-3-032-16955-6_5},
  isbn = {978-3-032-16955-6},
  langid = {english}
}

@inproceedings{odonoghue_making_2020,
  title = {Making Sense of Reinforcement Learning and Probabilistic Inference},
  booktitle = {International Conference on Learning Representations},
  author = {O'Donoghue, Brendan and Osband, Ian and Ionescu, Catalin},
  year = 2020,
  url = {https://openreview.net/forum?id=S1xitgHtvS}
}

@article{palmieri_unifying_2022,
  title = {A {{Unifying View}} of {{Estimation}} and {{Control Using Belief Propagation With Application}} to {{Path Planning}}},
  author = {Palmieri, Francesco A. N. and Pattipati, Krishna R. and Gennaro, Giovanni Di and Fioretti, Giovanni and Verolla, Francesco and Buonanno, Amedeo},
  year = 2022,
  journal = {IEEE Access},
  volume = {10},
  pages = {15193--15216},
  issn = {2169-3536},
  doi = {10.1109/ACCESS.2022.3148127}
}

@book{parr_active_2022,
  title = {Active {{Inference}}: {{The Free Energy Principle}} in {{Mind}}, {{Brain}}, and {{Behavior}}},
  shorttitle = {Active {{Inference}}},
  author = {Parr, Thomas and Pezzulo, Giovanni and Friston, Karl J.},
  year = 2022,
  month = mar,
  publisher = {The MIT Press},
  doi = {10.7551/mitpress/12441.001.0001},
  url = {https://direct.mit.edu/books/book/5299/Active-InferenceThe-Free-Energy-Principle-in-Mind},
  urldate = {2024-11-15},
  copyright = {https://creativecommons.org/licenses/by-nc-nd/4.0/},
  isbn = {978-0-262-36997-8},
  langid = {english}
}

@article{parr_generalised_2019,
  title = {Generalised Free Energy and Active Inference},
  author = {Parr, Thomas and Friston, Karl J.},
  year = 2019,
  month = dec,
  journal = {Biological Cybernetics},
  volume = {113},
  number = {5},
  pages = {495--513},
  issn = {1432-0770},
  doi = {10.1007/s00422-019-00805-w},
  url = {https://doi.org/10.1007/s00422-019-00805-w},
  urldate = {2024-07-31},
  langid = {english}
}

@article{paul_efficient_2024,
  title = {On Efficient Computation in Active Inference},
  author = {Paul, Aswin and Sajid, Noor and Costa, Lancelot Da and Razi, Adeel},
  year = 2024,
  month = nov,
  journal = {Expert Systems with Applications},
  volume = {253},
  eprint = {2307.00504},
  primaryclass = {cs},
  pages = {124315},
  issn = {0957-4174},
  doi = {10.1016/j.eswa.2024.124315},
  url = {http://arxiv.org/abs/2307.00504},
  urldate = {2025-05-01},
  archiveprefix = {arXiv}
}

@inproceedings{pearl_reverend_1982,
  title = {Reverend {{Bayes}} on {{Inference Engines}}: {{A Distributed Hierarchical Approach}}},
  shorttitle = {Reverend {{Bayes}} on {{Inference Engines}}},
  booktitle = {{{AAAI-82 Proceedings}}},
  author = {Pearl, Judea},
  year = 1982,
  pages = {133--136},
  publisher = {AAAI Press},
  address = {Carnegie Mellon University, Pittsburgh PA},
  url = {https://books.google.com/books?hl=nl&lr=&id=e59kEAAAQBAJ&oi=fnd&pg=PA129&ots=qrs53bNhtS&sig=auOq_v4YW1fTTgNOvsmydDN6RPY},
  urldate = {2025-05-06}
}

@article{rawlik_stochastic_2012,
  title = {On Stochastic Optimal Control and Reinforcement Learning by Approximate Inference},
  author = {Rawlik, Konrad and Toussaint, Marc and Vijayakumar, Sethu},
  year = 2012,
  journal = {Proceedings of Robotics: Science and Systems VIII},
  url = {https://books.google.com/books?hl=en&lr=&id=NOrxCwAAQBAJ&oi=fnd&pg=PA353&dq=On+stochastic+optimal+control+and+reinforcement+learning+by+approximate+inference&ots=dOLjhngmMZ&sig=k04S8PsFkrZCiTHocQskzN34wuk},
  urldate = {2025-05-04}
}

@book{russell_artificial_1995,
  title = {Artificial {{Intelligence}}: {{A}} Modern Approach},
  author = {Russell, Stuart and Norvig, Peter},
  year = 1995,
  publisher = {Prentice Hall},
  address = {Englewood Cliffs, NJ}
}

@phdthesis{senoz_message_2022,
  type = {Phd {{Thesis}} 1 ({{Research TU}}/e / {{Graduation TU}}/e)},
  title = {Message {{Passing Algorithms}} for {{Hierarchical Dynamical Models}}},
  author = {{\textcommabelow S}en{\"o}z, {\.I}smail},
  year = 2022,
  month = jun,
  address = {Eindhoven},
  isbn = {978-90-386-5532-1},
  school = {Eindhoven University of Technology}
}

@inproceedings{smith_heuristic_2004,
  title = {Heuristic Search Value Iteration for Pomdps},
  booktitle = {Proceedings of the Twentieth Conference Annual Conference on Uncertainty in Artificial Intelligence ({{UAI-04}}), Arlington, Virginia},
  author = {Smith, Trey and Simmons, Reid},
  year = 2004,
  pages = {520--527},
  publisher = {AUAI Press}
}

@book{sutton_reinforcement_2018,
  title = {Reinforcement Learning: {{An}} Introduction},
  shorttitle = {Reinforcement Learning},
  author = {Sutton, Richard S. and Barto, Andrew G.},
  year = 2018,
  publisher = {MIT press}
}

@inproceedings{tarbouriech_probabilistic_2023,
  title = {Probabilistic Inference in Reinforcement Learning Done Right},
  booktitle = {Advances in Neural Information Processing Systems},
  author = {Tarbouriech, Jean and Lattimore, Tor and O'Donoghue, Brendan},
  editor = {Oh, A. and Naumann, T. and Globerson, A. and Saenko, K. and Hardt, M. and Levine, S.},
  year = 2023,
  volume = {36},
  pages = {33687--33725},
  publisher = {Curran Associates, Inc.},
  url = {https://proceedings.neurips.cc/paper_files/paper/2023/file/6a6e010edde1b8f2812f558b67a1974e-Paper-Conference.pdf}
}

@inproceedings{todorov_general_2008,
  title = {General Duality between Optimal Control and Estimation},
  booktitle = {2008 47th {{IEEE Conference}} on {{Decision}} and {{Control}}},
  author = {Todorov, Emanuel},
  year = 2008,
  month = dec,
  pages = {4286--4292},
  issn = {0191-2216},
  doi = {10.1109/CDC.2008.4739438},
  url = {https://ieeexplore.ieee.org/abstract/document/4739438},
  urldate = {2025-05-04}
}

@inproceedings{todorov_linearlysolvable_2006,
  title = {Linearly-Solvable {{Markov}} Decision Problems},
  booktitle = {Advances in {{Neural Information Processing Systems}}},
  author = {Todorov, Emanuel},
  year = 2006,
  volume = {19},
  publisher = {MIT Press},
  url = {https://proceedings.neurips.cc/paper_files/paper/2006/hash/d806ca13ca3449af72a1ea5aedbed26a-Abstract.html},
  urldate = {2025-05-01}
}

@inproceedings{toussaint_robot_2009,
  title = {Robot Trajectory Optimization Using Approximate Inference},
  booktitle = {Proceedings of the 26th {{Annual International Conference}} on {{Machine Learning}}},
  author = {Toussaint, Marc},
  year = 2009,
  month = jun,
  pages = {1049--1056},
  publisher = {ACM},
  address = {Montreal Quebec Canada},
  doi = {10.1145/1553374.1553508},
  url = {https://dl.acm.org/doi/10.1145/1553374.1553508},
  urldate = {2025-05-04},
  isbn = {978-1-60558-516-1},
  langid = {english}
}

@inproceedings{towers_gymnasium_2025,
  title = {Gymnasium: A Standard Interface for Reinforcement Learning Environments},
  booktitle = {Advances in Neural Information Processing Systems},
  author = {Towers, Mark and Kwiatkowski, Ariel and Balis, John and De Cola, Gianluca and Deleu, Tristan and Goul{\~a}o, Manuel and Andreas, Kallinteris and Krimmel, Markus and KG, Arjun and {Perez-Vicente}, Rodrigo and Terry, J and Pierr{\'e}, Andrea and Schulhoff, Sander and Tai, Jun Jet and Tan, Hannah and Younis, Omar G.},
  editor = {Belgrave, D. and Zhang, C. and Lin, H. and Pascanu, R. and Koniusz, P. and Ghassemi, M. and Chen, N.},
  year = 2025,
  volume = {38},
  publisher = {Curran Associates, Inc.},
  url = {https://proceedings.neurips.cc/paper_files/paper/2025/file/d7ff1795e8527f6443371c3933bdb52b-Paper-Datasets_and_Benchmarks_Track.pdf}
}

@article{vandelaar_realizing_2024,
  title = {Realizing Synthetic Active Inference Agents, Part {{II}}: {{Variational}} Message Updates},
  author = {{van de Laar}, Thijs and Koudahl, Magnus and {de Vries}, Bert},
  year = 2024,
  journal = {Neural Computation},
  volume = {37},
  number = {1},
  pages = {38--75},
  publisher = {MIT Press 255 Main Street, 9th Floor, Cambridge, Massachusetts 02142, USA \dots}
}

@article{wainwright_graphical_2008,
  title = {Graphical Models, Exponential Families, and Variational Inference},
  author = {Wainwright, Martin J. and Jordan, Michael I.},
  year = 2008,
  journal = {Foundations and Trends\textregistered{} in Machine Learning},
  volume = {1},
  number = {1--2},
  pages = {1--305},
  publisher = {Now Publishers, Inc.},
  url = {https://www.nowpublishers.com/article/Details/MAL-001},
  urldate = {2025-05-06}
}

@article{yedidia_constructing_2005,
  title = {Constructing Free-Energy Approximations and Generalized Belief Propagation Algorithms},
  author = {Yedidia, Jonathan S. and Freeman, W.T. and Weiss, Y.},
  year = 2005,
  month = jul,
  journal = {IEEE Transactions on Information Theory},
  volume = {51},
  number = {7},
  pages = {2282--2312},
  issn = {0018-9448},
  doi = {10.1109/TIT.2005.850085}
}

\newpage
\onecolumn

\title{What Type of Inference is Active Inference?\\(Supplementary Material)}
\maketitle

\appendix
\numberwithin{equation}{section}

\section{Entropy Corrections}
\label{appx:entropy_corrections}

\subsection{Entropy Correction for Planning}
\label{appx:planning_entropy_correction}

We derive the entropy correction that distinguishes planning-as-inference from marginal inference.
\citet{lazaro-gredilla_what_2024} formulate planning-as-inference using a ``planning entropy'' that excludes action variables from the trajectory entropy.
Here we show that this formulation is equivalent to adding an entropy correction to the standard VFE, and derive the form of this correction.

\subsubsection{The Planning Entropy of \texorpdfstring{\citet{lazaro-gredilla_what_2024}}{Lazaro-Gredilla et al. (2024)}}

\citet{lazaro-gredilla_what_2024} define the planning entropy as:
\begin{equation}\label{eq:lazaro_planning_entropy}
    \ent{q(x_0)} + \sum_{t=1}^T \mathbb{H}_{q}[x_t | x_{t-1}, u_t]\,,
\end{equation}
where $\mathbb{H}_{q}[x_t | x_{t-1}, u_t] = -\int q(x_t, x_{t-1}, u_t) \log q(x_t | x_{t-1}, u_t) \dif x_t \dif x_{t-1} \dif u_t$ denotes the conditional entropy.
This differs from the full trajectory entropy $\ent{q(\bm{x}, \bm{u})}$ by excluding the action entropy.

\begin{remark}[Notation conventions]
    We use $u_t$ to denote the action that leads to state $x_t$, following the convention in this paper.
    \citet{lazaro-gredilla_what_2024} use a different indexing where $u_t$ leads to $x_{t+1}$.
    Additionally, they formulate their objective as a maximization problem (maximizing the variational bound), whereas we minimize the VFE; this flips the sign of entropy terms.
\end{remark}

\subsubsection{Derivation of the Entropy Correction}

\begin{proposition}[Planning entropy decomposition] \label{prop:planning_entropy_decomposition}
    The planning entropy~\eqref{eq:lazaro_planning_entropy} equals the full trajectory entropy plus an entropy correction:
    \begin{equation}
        \ent{q(x_0)} + \sum_{t=1}^T \mathbb{H}_{q}[x_t | x_{t-1}, u_t] = \ent{q(\bm{x}, \bm{u})} + \sum_{t=1}^T \ent{q(x_{t-1})} - \ent{q(x_{t-1}, u_t)}\,.
    \end{equation}
\end{proposition}
\begin{proof}
    Starting from the planning entropy and expanding the conditional entropy:
    \begin{align}
        \ent{q(x_0)} & + \sum_{t=1}^T \mathbb{H}_{q}[x_t | x_{t-1}, u_t] \notag                                                                                                               \\
                     & = \ent{q(x_0)} + \sum_{t=1}^T \left( -\iiint q(x_t, x_{t-1}, u_t) \log q(x_t | x_{t-1}, u_t) \dif x_t \dif x_{t-1} \dif u_t \right)\,,                                    \\
                     & = \ent{q(x_0)} + \sum_{t=1}^T \left( -\iiint q(x_t, x_{t-1}, u_t) \log \frac{q(x_t, x_{t-1}, u_t)}{q(u_t | x_{t-1}) q(x_{t-1})} \dif x_t \dif x_{t-1} \dif u_t \right)\,.
    \end{align}
    Splitting the logarithm:
    \begin{align}
         & = \ent{q(x_0)} + \sum_{t=1}^T \left( -\iiint q(x_t, x_{t-1}, u_t) \log \frac{q(x_t, x_{t-1}, u_t)}{q(x_{t-1})} \dif x_t \dif x_{t-1} \dif u_t \right. \notag                               \\
         & \qquad \left. + \iiint q(x_t, x_{t-1}, u_t) \log q(u_t | x_{t-1}) \dif x_t \dif x_{t-1} \dif u_t \right)\,,                                                                                   \\
         & = \ent{q(x_0)} + \sum_{t=1}^T \underbrace{\left( -\iiint q(x_t, x_{t-1}, u_t) \log q(x_t, u_t | x_{t-1}) \dif x_t \dif x_{t-1} \dif u_t \right)}_{\mathbb{H}_q[x_t, u_t | x_{t-1}]} \notag \\
         & \qquad + \sum_{t=1}^T \iint q(x_{t-1}, u_t) \log \frac{q(x_{t-1}, u_t)}{q(x_{t-1})} \dif x_{t-1} \dif u_t\,.
    \end{align}
    The first sum gives the trajectory entropy by the chain rule:
    \begin{equation}
        \ent{q(x_0)} + \sum_{t=1}^T \mathbb{H}_q[x_t, u_t | x_{t-1}] = \ent{q(\bm{x}, \bm{u})}\,.
    \end{equation}
    The second sum expands as:
    \begin{align}
        \sum_{t=1}^T \iint q(x_{t-1}, u_t) \log \frac{q(x_{t-1}, u_t)}{q(x_{t-1})} \dif x_{t-1} \dif u_t
         & = \sum_{t=1}^T \underbrace{\iint q(x_{t-1}, u_t) \log q(x_{t-1}, u_t) \dif x_{t-1} \dif u_t}_{-\ent{q(x_{t-1}, u_t)}} \notag \\
         & \qquad - \underbrace{\int q(x_{t-1}) \log q(x_{t-1}) \dif x_{t-1}}_{-\ent{q(x_{t-1})}}\,,                                       \\
         & = \sum_{t=1}^T \ent{q(x_{t-1})} - \ent{q(x_{t-1}, u_t)}\,.
    \end{align}
    Combining these results proves the proposition.
\end{proof}

\subsubsection{Interpretation}

Since we minimize the VFE (rather than maximize as in \citet{lazaro-gredilla_what_2024}), the planning entropy is \emph{subtracted} from the objective.
This means the entropy correction $\sum_t \ent{q(x_{t-1}, u_t)} - \ent{q(x_{t-1})} = \sum_t \ent{q(u_t | x_{t-1})}$ is \emph{added} to the VFE.

Adding this positive correction penalizes action uncertainty: since we minimize the objective, high $\ent{q(u_t | x_{t-1})}$ increases the cost, pushing the agent toward a deterministic policy.

\subsection{Proof of Theorem~\ref{thm:entropy_decomposition}}
\label{appx:entropy_decomposition_proof}

We prove that the VFE of the augmented model~\eqref{eq:augmented_model} decomposes into the original VFE plus entropy correction terms.
The proof requires three lemmas, each showing how one epistemic prior contributes to the entropy correction.

\begin{lemma}[State epistemic prior contribution] \label{lem:p_tilde_x}
    Let $q(\bm{y}, \bm{x}, \bm{u}, \theta)$ be a variational distribution over the generative model~\eqref{eq:generative_model}, and let the state epistemic prior be defined as in~\eqref{eq:epistemic_x}:
    \begin{equation}
        \tilde{p}(x_t) = \exp \bigl(\Ex{q(\theta | x_t)}{-\hent{q(y_t | x_t, \theta)}}\bigr)\,.
    \end{equation}
    Then:
    \begin{equation}
        -\int q(x_{t}) \log \tilde{p}(x_t) \dif x_t = \ent{q(y_t | x_t, \theta)}\,.
    \end{equation}
\end{lemma}
\begin{proof}
    Substituting the definition of $\tilde{p}(x_t)$ and expanding the conditional entropy:
    \begin{align}
        -\int q(x_{t}) \log \tilde{p}(x_t) \dif x_t
         & = -\int q(x_{t}) \int q(\theta | x_t) \int q(y_t | x_t, \theta) \log q(y_t | x_t, \theta) \dif y_t \dif \theta \dif x_t \\
         & = -\iiint q(y_t | x_t, \theta)\, q(\theta | x_t)\, q(x_{t}) \log q(y_t | x_t, \theta) \dif y_t \dif \theta \dif x_t     \\
         & = -\iiint q(y_t, x_{t}, \theta) \log q(y_t | x_t, \theta) \dif y_t \dif x_t \dif \theta                                 \\
         & =  \ent{q(y_t | x_t, \theta)}\,.
    \end{align}
\end{proof}

\begin{lemma}[Action epistemic prior contribution] \label{lem:p_tilde_u}
    Let $q(\bm{y}, \bm{x}, \bm{u}, \theta)$ be a variational distribution over the generative model~\eqref{eq:generative_model}, and let the action epistemic prior be defined as in~\eqref{eq:epistemic_u}:
    \begin{equation}
        \tilde{p}(u_t) = \exp\bigl(\hent{q(x_t, x_{t-1} | u_t)} - \hent{q(x_{t-1} | u_t)}\bigr)\,.
    \end{equation}
    Then:
    \begin{equation}
        -\int q(u_t) \log \tilde{p}(u_t) \dif u_t = -\ent{q(x_t | x_{t-1}, u_t)}\,.
    \end{equation}
\end{lemma}
\begin{proof}
    Substituting the definition of $\tilde{p}(u_t)$:
    \begin{align}
        -\int q(u_t) & \log \tilde{p}(u_t) \dif u_t \notag                                                                                                          \\
                     & = \int q(u_t) \left( \iint q(x_t, x_{t-1} | u_t) \log q(x_t, x_{t-1} | u_t) \dif x_t \dif x_{t-1} \right. \notag                             \\
                     & \qquad \left. - \int q(x_{t-1} | u_t) \log q(x_{t-1} | u_t) \dif x_{t-1} \right) \dif u_t                                                    \\
                     & = \int q(u_t) \left( \iint \frac{q(x_t, x_{t-1}, u_t)}{q(u_t)} \log \frac{q(x_t, x_{t-1}, u_t)}{q(u_t)} \dif x_t \dif x_{t-1} \right. \notag \\
                     & \qquad \left. - \int \frac{q(x_{t-1}, u_t)}{q(u_t)} \log \frac{q(x_{t-1}, u_t)}{q(u_t)} \dif x_{t-1} \right) \dif u_t                        \\
                     & = \iiint q(x_t, x_{t-1}, u_t) \log \frac{q(x_t, x_{t-1}, u_t)}{q(u_t)} \dif x_t \dif x_{t-1} \dif u_t \notag                                 \\
                     & \qquad - \iint q(x_{t-1}, u_t) \log \frac{q(x_{t-1}, u_t)}{q(u_t)} \dif x_{t-1} \dif u_t                                                     \\
                     & = -\ent{q(x_t, x_{t-1}, u_t)} + \ent{q(u_t)} + \ent{q(x_{t-1}, u_t)} - \ent{q(u_t)}                                                          \\
                     & = \ent{q(x_{t-1}, u_t)} - \ent{q(x_t, x_{t-1}, u_t)} = -\ent{q(x_t | x_{t-1}, u_t)}\,.
    \end{align}
\end{proof}

\begin{lemma}[Observation epistemic prior contribution] \label{lem:p_tilde_x_y}
    Let $q(\bm{y}, \bm{x}, \bm{u}, \theta)$ be a variational distribution over the generative model~\eqref{eq:generative_model}, and let the observation epistemic prior be defined as in~\eqref{eq:epistemic_xy}:
    \begin{equation}
        \tilde{p}(y_t, x_t) = \exp\bigl(\KL{q(\theta | y_t, x_t)}{q(\theta | x_t)}\bigr)\,.
    \end{equation}
    Then:
    \begin{equation}
        -\iint q(y_t, x_t) \log \tilde{p}(y_t, x_t) \dif y_t \dif x_t = \ent{q(y_t | x_t, \theta)} - \ent{q(y_t | x_t)}\,.
    \end{equation}
\end{lemma}
\begin{proof}
    Substituting the definition of $\tilde{p}(y_t, x_t)$:
    \begin{align}
        -\iint q(y_t, x_t) & \log \tilde{p}(y_t, x_t) \dif y_t \dif x_t \notag                                                                                                                            \\
                           & = -\iint q(y_t, x_t) \left( \int q(\theta | y_t, x_t) \log \frac{q(\theta | y_t, x_t)}{q(\theta | x_t)} \dif\theta \right) \dif y_t \dif x_t                                 \\
                           & = -\iint q(y_t, x_t) \left( \int q(\theta | y_t, x_t) \log \frac{q(y_t, x_t, \theta)}{q(y_t, x_t)} - \log \frac{q(x_t, \theta)}{q(x_t)} \dif\theta \right) \dif y_t \dif x_t \\
                           & = -\iiint q(y_t, x_t, \theta) \left( \log q(y_t, x_t, \theta) - \log q(y_t, x_t) - \log q(x_t, \theta) + \log q(x_t) \right) \dif y_t \dif x_t \dif\theta                    \\
                           & = \underbrace{-\iiint q(y_t, x_t, \theta) \log q(y_t, x_t, \theta) \dif y_t \dif x_t \dif\theta}_{\ent{q(y_t, x_t, \theta)}} \notag                                          \\
                           & \qquad + \underbrace{\iiint q(y_t, x_t, \theta) \log q(y_t, x_t) \dif y_t \dif x_t \dif\theta}_{-\ent{q(y_t, x_t)}} \notag                                                   \\
                           & \qquad + \underbrace{\iiint q(y_t, x_t, \theta) \log q(x_t, \theta) \dif y_t \dif x_t \dif\theta}_{-\ent{q(x_t, \theta)}} \notag                                             \\
                           & \qquad - \underbrace{\iiint q(y_t, x_t, \theta) \log q(x_t) \dif y_t \dif x_t \dif\theta}_{\ent{q(x_t)}}                                                                     \\
                           & = \ent{q(y_t, x_t, \theta)} - \ent{q(y_t, x_t)} - \ent{q(x_t, \theta)} + \ent{q(x_t)}                                                                                        \\
                           & = \bigl(\ent{q(y_t, x_t, \theta)} - \ent{q(x_t, \theta)}\bigr) - \bigl(\ent{q(y_t, x_t)} - \ent{q(x_t)}\bigr)                                                                \\
                           & = \ent{q(y_t | x_t, \theta)} - \ent{q(y_t | x_t)}\,.
    \end{align}
\end{proof}

\begin{proof}[Proof of Theorem~\ref{thm:entropy_decomposition}]
    The VFE of the augmented model~\eqref{eq:augmented_model} is:
    \begin{align}
        F_{\tilde{p}}[q] & = \int q(\bm{y}, \bm{x}, \bm{u}, \theta) \log \frac{q(\bm{y}, \bm{x}, \bm{u}, \theta)}{\tilde{p}(\bm{y}, \bm{x}, \bm{u}, \theta)}                                                                                                                                   \\
                         & = \int q(\bm{y}, \bm{x}, \bm{u}, \theta) \log \frac{q(\bm{y}, \bm{x}, \bm{u}, \theta)}{p(\bm{y}, \bm{x}, \bm{u}, \theta) \prod_{t=1}^T \tilde{p}(x_t) \tilde{p}(u_t) \tilde{p}(y_t, x_t)}                                                                           \\
                         & = \underbrace{\int q(\bm{y}, \bm{x}, \bm{u}, \theta) \log \frac{q(\bm{y}, \bm{x}, \bm{u}, \theta)}{p(\bm{y}, \bm{x}, \bm{u}, \theta)}}_{F_{\hat{p}}[q]} + \notag                                                                                                    \\
                         & \qquad - \sum_{t=1}^T \Bigg( \iiiint q(\bm{y}, \bm{x}, \bm{u}, \theta) \log \tilde{p}(x_t) \dif \bm{y} \dif \bm{x} \dif \bm{u} \dif \theta + \iiiint q(\bm{y}, \bm{x}, \bm{u}, \theta) \log \tilde{p}(u_t) \dif \bm{y} \dif \bm{x} \dif \bm{u} \dif \theta + \notag \\
                         & \qquad + \iiiint q(\bm{y}, \bm{x}, \bm{u}, \theta) \log \tilde{p}(y_t, x_t) \dif \bm{y} \dif \bm{x} \dif \bm{u} \dif \theta \Bigg)                                                                                                                                  \\
                         & = F_{\hat{p}}[q] + \sum_{t=1}^T \left( -\int q(x_t) \log \tilde{p}(x_t) \dif x_t - \int q(u_t) \log \tilde{p}(u_t) \dif u_t - \iint q(y_t, x_t) \log \tilde{p}(y_t, x_t) \dif y_t \dif x_t \right)\,.
    \end{align}
    Applying Lemmas~\ref{lem:p_tilde_x}--\ref{lem:p_tilde_x_y}:
    \begin{align}
        F_{\tilde{p}}[q] & = F_{\hat{p}}[q] + \sum_{t=1}^{T} \Big( \underbrace{\ent{q(y_t | x_t, \theta)}}_{\text{\lref{lem:p_tilde_x}}} \underbrace{{}-\ent{q(x_t | x_{t-1}, u_t)}}_{\text{\lref{lem:p_tilde_u}}} + \underbrace{\ent{q(y_t | x_t, \theta)} - \ent{q(y_t | x_t)}}_{\text{\lref{lem:p_tilde_x_y}}} \Big) \\
                         & = F_{\hat{p}}[q] + \sum_{t=1}^{T} 2\ent{q(y_t | x_t, \theta)} - \ent{q(x_t | x_{t-1}, u_t)} - \ent{q(y_t | x_t)}\,.
    \end{align}
\end{proof}

\subsection{Cross-Entropy Interpretation}
\label{appx:planning_cross_entropy}

We show that planning-as-inference from \citet{lazaro-gredilla_what_2024} is equivalent to minimizing cross-entropy to preference distributions.

\subsubsection{Reward as Cross-Entropy}

\begin{proposition}[Reward as cross-entropy] \label{prop:reward_cross_entropy}
    For any $\lambda > 0$, define $\lambda$-scaled preference distributions $\hat{p}_\lambda(x_t) \propto \exp(\lambda R_t(x_t))$.
    Then for any distribution $q(\bm{x}, \bm{u})$:
    \begin{equation}
        \Ex{q(\bm{x}, \bm{u})}{\sum_{t=1}^T R_t(x_t)} = -\frac{1}{\lambda}\sum_{t=1}^T \xent{q(x_t)}{\hat{p}_\lambda(x_t)} + \text{const}\,.
    \end{equation}
\end{proposition}
\begin{proof}
    With $\log \hat{p}_\lambda(x_t) = \lambda R_t(x_t) - \log Z_{t,\lambda}$ where $Z_{t,\lambda} = \int \exp(\lambda R_t(x_t)) \dif x_t$:
    \begin{align}
        \xent{q(x_t)}{\hat{p}_\lambda(x_t)} & = -\Ex{q(x_t)}{\log \hat{p}_\lambda(x_t)}               \\
                                            & = -\Ex{q(x_t)}{\lambda R_t(x_t) - \log Z_{t,\lambda}}   \\
                                            & = -\lambda\Ex{q(x_t)}{R_t(x_t)} + \log Z_{t,\lambda}\,.
    \end{align}
    Rearranging: $\Ex{q(x_t)}{R_t(x_t)} = -\frac{1}{\lambda}\xent{q(x_t)}{\hat{p}_\lambda(x_t)} + \frac{1}{\lambda}\log Z_{t,\lambda}$.
    Summing over $t$ gives the result, where $\frac{1}{\lambda}\sum_t \log Z_{t,\lambda}$ is constant with respect to $q$.
\end{proof}

\subsubsection{Connection to Planning-as-Inference}

\citet[Theorem 1]{lazaro-gredilla_what_2024} show that the best exponential utility
\begin{equation}
    F_\lambda^{\text{planning}} = \frac{1}{\lambda} \log \Ex{p(\bm{x}, \bm{u})}{\exp\left(\lambda \sum_{t=1}^T R_t\right)}
\end{equation}
can be expressed as the result of a variational optimization problem whose objective includes the expected sum of rewards $\Ex{q}{\sum_t R_t}$ as one of its terms.

By \refprop{prop:reward_cross_entropy} with $\hat{p}_\lambda(x_t) \propto \exp(\lambda R_t(x_t))$, this reward term equals (up to a constant) $-\frac{1}{\lambda}\sum_t \xent{q(x_t)}{\hat{p}_\lambda(x_t)}$.
Since the variational bound in \citet[Theorem 1]{lazaro-gredilla_what_2024} has an overall $\frac{1}{\lambda}$ scaling, this factor cancels, and maximizing $F_\lambda^{\text{planning}}$ is equivalent to minimizing $\sum_t \xent{q(x_t)}{\hat{p}_\lambda(x_t)}$ along with the dynamics and entropy terms.

This establishes that with $\hat{p}_\lambda(x_t) \propto \exp(\lambda R_t(x_t))$, the expected utility becomes cross-entropy to preference distributions.
Consequently, the inference procedure from \citet{lazaro-gredilla_what_2024} minimizes cross-entropy by optimizing its variational objective.
Different values of $\lambda$ yield different preference distributions (corresponding to different risk attitudes), but the underlying mechanism remains cross-entropy minimization.


\section{EFE-based Planning Inference}
\label{appx:efe_planning_inference}

\begin{theorem}[EFE-based planning inference]\label{thm:efe-planning}
    Consider the augmented model
    $\tilde{p}(\bm{y}, \bm{x}, \bm{u}, \theta)$
    from~\eqref{eq:augmented_model} and a set of reactive
    policies $\bm{\pi} = \{\pi_t(u_t | x_{t-1})\}_{t=1}^T$.
    Then:
    \begin{equation}\label{eq:efe-planning-result}
        \max_{\bm{\pi}}
        \log \sum_{\bm{y}, \bm{x}, \bm{u}, \theta}
        \tilde{p}(\bm{y}, \bm{x}, \bm{u}, \theta)
        \prod_{t=1}^T \pi_t(u_t | x_{t-1})
        = \max_q\;
        \bigl\langle \log \tilde{p} \bigr\rangle_{q}
        + \ent{q(\bm{y}, \bm{x}, \bm{u}, \theta)}
        - \sum_{t=1}^T \ent{q(u_t | x_{t-1})}\,,
    \end{equation}
    where $q(\bm{y}, \bm{x}, \bm{u}, \theta)$ is an arbitrary
    variational distribution and the optimal policy is
    $\pi_t^*(u_t | x_{t-1}) = q^*(u_t | x_{t-1})$.
    Equivalently, up to an additive constant independent of $q$:
    \begin{equation}\label{eq:efe-planning-vfe}
        \max_{\bm{\pi}}
        \log \sum_{\bm{y}, \bm{x}, \bm{u}, \theta}
        \tilde{p}(\bm{y}, \bm{x}, \bm{u}, \theta)
        \prod_{t=1}^T \pi_t(u_t | x_{t-1})
        =
        - \min_q\;
        F_{\tilde{p}}[q]
        + \sum_{t=1}^T \ent{q(u_t | x_{t-1})}\,,
    \end{equation}
    where $F_{\tilde{p}}[q] = \KL{q}{\tilde{p}}$.
\end{theorem}

\begin{proof}
    This is the variational formulation of planning from
    \citet[Appendix~A]{lazaro-gredilla_what_2024}, rewritten in the
    notation of this paper. To connect our formulation to theirs, note
    that, up to a normalization constant independent of $\bm{\pi}$,
    \begin{align}
        & \sum_{\bm{y}, \bm{x}, \bm{u}, \theta}
        \tilde{p}(\bm{y}, \bm{x}, \bm{u}, \theta)
        \prod_{t=1}^T \pi_t(u_t | x_{t-1}) \notag \\
        & \qquad \propto
        \sum_{\bm{y}, \bm{x}, \bm{u}, \theta}
        \rho_{\bm{\pi}}(\bm{y}, \bm{x}, \bm{u}, \theta)
        \prod_{t=1}^T \hat{p}(x_t)\hat{p}(y_t)\,,
        \label{eq:efe-planning-bridge}
    \end{align}
    where
    \begin{equation}
        \rho_{\bm{\pi}}(\bm{y}, \bm{x}, \bm{u}, \theta)
        := p(\theta)\, p(x_0)\!
        \prod_{t=1}^T
        p(y_t | x_t, \theta)\,
        p(x_t | x_{t-1}, u_t, \theta)\,
        p(u_t)\,
        \tilde{p}(u_t)\,
        \tilde{p}(x_t)\,
        \tilde{p}(y_t, x_t)\,
        \pi_t(u_t | x_{t-1})\,.
    \end{equation}
    Thus the objective is the log evidence of a preference-weighted
    rollout model, with $\log \hat{p}(x_t)$ and $\log \hat{p}(y_t)$
    playing the role of rewards, exactly as in
    \refappx{appx:planning_cross_entropy}. The remaining changes
    relative to \citet{lazaro-gredilla_what_2024} are that: (i) actions are indexed by the
    state they lead to, so policies are written as $\pi_t(u_t|x_{t-1})$;
    and (ii) the planning entropy differs from the full trajectory
    entropy by the correction $\sum_t \ent{q(u_t | x_{t-1})}$ derived in
    \refappx{appx:planning_entropy_correction}. The optimal policy
    satisfies $\pi_t^*(u_t | x_{t-1}) = q^*(u_t | x_{t-1})$ by the same
    maximization over policies.
\end{proof}

\begin{corollary}[Combined entropy correction]\label{cor:combined-correction}
    Applying \refthm{thm:entropy_decomposition} to expand
    $F_{\tilde{p}}[q]$, the objective~\eqref{eq:efe-planning-vfe}
    becomes
    \begin{equation}\label{eq:combined-correction}
        \min_q\;
        F_{\hat{p}}[q]
        + \underbrace{
            \sum_{t=1}^T
            2\ent{q(y_t | x_t, \theta)}
            - \ent{q(x_t | x_{t-1}, u_t)}
            - \ent{q(y_t | x_t)}
        }_{\Delta^{\mathrm{AIF}}}
        + \underbrace{
            \sum_{t=1}^T \ent{q(u_t | x_{t-1})}
        }_{\Delta^{\mathrm{planning}}}\,.
    \end{equation}
\end{corollary}
\begin{proof}
    Direct substitution of \refthm{thm:entropy_decomposition}
    into~\eqref{eq:efe-planning-vfe}.
\end{proof}

\section{Bethe Free Energy Details}
\label{appx:bethe}

This appendix reviews the standard derivation of the Bethe Free Energy (BFE) from the Variational Free Energy (VFE) on factor graphs, following \citet{yedidia_constructing_2005} and \citet{wainwright_graphical_2008}.
The material is collected here for self-containedness and to establish the notation used in \Cref{subsec:bethe} and subsequent appendices.

\subsection{Factorized Generative Model on an FFG}
\label{appx:bethe_ffg}

\begin{figure}[t]
    \centering
    \includegraphics[width=0.3\columnwidth]{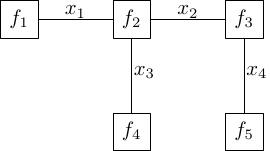}
    \caption{A Forney-style factor graph representing the factorization $p(\bm{x}) = f_1(x_1)\, f_2(x_1, x_2, x_3)\, f_3(x_2, x_4)\, f_4(x_3)\, f_5(x_4)$. Square nodes denote factors; edges denote variables.}
    \label{fig:ffg_example}
\end{figure}

Consider a generative model that factorizes as
\begin{equation}\label{eq:appx-model-factorization}
    p(\bm{s}) = \prod_{a \in \mathcal{V}} f_a(\bm{s}_a)\,,
\end{equation}
where each factor $f_a$ depends on a subset of variables $\bm{s}_a \subseteq \bm{s}$.
This factorization is represented as a Forney-style factor graph (FFG) $\mathcal{G} = (\mathcal{V}, \mathcal{E})$~\citep{forney_codes_2001,loeliger_factor_2007} (see \Cref{fig:ffg_example}), where $\mathcal{V}$ denotes the set of factor nodes and $\mathcal{E}$ the set of edges (variables).
An edge $i \in \mathcal{E}$ connects to a node $a \in \mathcal{V}$ whenever variable $s_i$ appears in the scope of $f_a$.
We write $\mathcal{E}(a)$ for the edges adjacent to node $a$ and $\mathcal{V}(i)$ for the nodes adjacent to edge $i$.
The \emph{degree} of edge $i$ is $d_i = |\mathcal{V}(i)|$, counting the number of factors in which $s_i$ participates.

\subsection{The Bethe Variational Family}
\label{appx:bethe_family}

The Bethe approximation~\citep{yedidia_constructing_2005} parameterizes the variational distribution in terms of \emph{local beliefs}: a factor belief $q_a(\bm{s}_a)$ for each node $a \in \mathcal{V}$ and a singleton belief $q_i(s_i)$ for each edge $i \in \mathcal{E}$.
These beliefs define a pseudo-distribution via the \emph{Bethe factorization}:
\begin{equation}\label{eq:appx-bethe-factorization}
    q_{\mathrm{Bethe}}(\bm{s}) = \frac{\prod_{a \in \mathcal{V}} q_a(\bm{s}_a)}{\prod_{i \in \mathcal{E}} q_i(s_i)^{d_i - 1}}\,.
\end{equation}
The denominator corrects the over-counting: since each variable $s_i$ appears in $d_i$ factor beliefs, naively multiplying all $q_a$ would count $q_i$ a total of $d_i$ times.
Dividing by $q_i^{d_i - 1}$ reduces this to a single effective copy.
For a variable that appears in only one factor ($d_i = 1$), no correction is needed.

This parameterization is valid only when the beliefs satisfy \emph{local consistency constraints}:
\begin{subequations}\label{eq:appx-local-consistency}
    \begin{align}
        \int q_a(\bm{s}_a) \, \dif \bm{s}_{a \setminus i} & = q_i(s_i) \quad \text{for all } a \in \mathcal{V},\; i \in \mathcal{E}(a)\,, \label{eq:appx-marginalization} \\
        \int q_a(\bm{s}_a) \, \dif \bm{s}_a               & = 1 \quad \text{for all } a \in \mathcal{V}\,. \label{eq:appx-bethe-normalization}
    \end{align}
\end{subequations}
The marginalization constraint~\eqref{eq:appx-marginalization} requires that each factor belief, when marginalized over all variables except $s_i$, agrees with the singleton belief $q_i$.
Together with normalization~\eqref{eq:appx-bethe-normalization}, these constraints define the \emph{local polytope} $\mathcal{L}_\mathcal{G}$.

\subsection{Derivation: VFE to BFE}
\label{appx:vfe_to_bfe}

Substituting the model factorization~\eqref{eq:appx-model-factorization} and the Bethe factorization~\eqref{eq:appx-bethe-factorization} into the VFE yields the BFE.

\paragraph{Step 1: Expand the log terms.}
Taking the logarithm of the Bethe factorization:
\begin{equation}\label{eq:appx-log-q}
    \log q_{\mathrm{Bethe}}(\bm{s}) = \sum_{a \in \mathcal{V}} \log q_a(\bm{s}_a) - \sum_{i \in \mathcal{E}} (d_i - 1) \log q_i(s_i)\,.
\end{equation}
Similarly, the log model decomposes as:
\begin{equation}\label{eq:appx-log-p}
    \log p(\bm{s}) = \sum_{a \in \mathcal{V}} \log f_a(\bm{s}_a)\,.
\end{equation}

\paragraph{Step 2: Substitute into the VFE.}
The VFE is $F[q] = \int q(\bm{s}) \log \frac{q(\bm{s})}{p(\bm{s})} \, \dif \bm{s}$.
Substituting~\eqref{eq:appx-log-q} and~\eqref{eq:appx-log-p}:
\begin{equation}\label{eq:appx-vfe-expanded}
    F[q] = \int q(\bm{s}) \left[ \sum_{a \in \mathcal{V}} \log q_a(\bm{s}_a) - \sum_{i \in \mathcal{E}} (d_i - 1) \log q_i(s_i) - \sum_{a \in \mathcal{V}} \log f_a(\bm{s}_a) \right] \dif \bm{s}\,.
\end{equation}

\paragraph{Step 3: Localize expectations using consistency.}
The key step exploits the local consistency constraints.
For any function $g(\bm{s}_a)$ that depends only on the variables in the scope of factor $a$:
\begin{equation}\label{eq:appx-localization-factor}
    \int q(\bm{s}) \, g(\bm{s}_a) \, \dif \bm{s} = \int q_a(\bm{s}_a) \, g(\bm{s}_a) \, \dif \bm{s}_a\,,
\end{equation}
since $q$ marginalizes to $q_a$ over $\bm{s}_a$ by consistency.
Similarly, for any function $h(s_i)$ of a single variable:
\begin{equation}\label{eq:appx-localization-singleton}
    \int q(\bm{s}) \, h(s_i) \, \dif \bm{s} = \int q_i(s_i) \, h(s_i) \, \dif s_i\,.
\end{equation}

Applying these identities to~\eqref{eq:appx-vfe-expanded}:
\begin{equation}\label{eq:appx-bfe-intermediate}
    F[q] = \sum_{a \in \mathcal{V}} \int q_a(\bm{s}_a) \log \frac{q_a(\bm{s}_a)}{f_a(\bm{s}_a)} \, \dif \bm{s}_a - \sum_{i \in \mathcal{E}} (d_i - 1) \int q_i(s_i) \log q_i(s_i) \, \dif s_i\,.
\end{equation}

\paragraph{Step 4: Identify the BFE.}
Recognizing the KL divergence and entropy, one obtains the \emph{Bethe Free Energy}:
\begin{equation}\label{eq:appx-bfe}
    F_{\text{Bethe}}[q] = \sum_{a \in \mathcal{V}} \KL{q_a(\bm{s}_a)}{f_a(\bm{s}_a)} + \sum_{i \in \mathcal{E}} (d_i - 1) \, \ent{q_i(s_i)}\,,
\end{equation}
which is \eqref{eq:bethe_free_energy} in the main text.
The first sum penalizes each factor belief for deviating from its corresponding factor, while the second sum adds back the entropy of shared variables to correct for the over-counting inherent in the Bethe factorization.

\subsection{Constrained Optimization and Belief Propagation}
\label{appx:bethe_bp}

The BFE gives rise to the constrained optimization problem
\begin{equation}\label{eq:appx-cbfe}
    \min_{\{q_a, q_i\} \in \mathcal{L}_\mathcal{G}} F_{\text{Bethe}}[q]\,,
\end{equation}
where $\mathcal{L}_\mathcal{G}$ is the local polytope defined by the constraints~\eqref{eq:appx-local-consistency}.
\citet{yedidia_constructing_2005} showed that the stationary points of this constrained problem correspond exactly to the fixed points of the \emph{belief propagation} (BP) algorithm.

On \emph{tree-structured} graphs, the BFE equals the exact VFE, and BP converges to the exact posterior marginals~\citep{pearl_reverend_1982}.
In this case, the Bethe factorization~\eqref{eq:appx-bethe-factorization} is an exact representation of the global posterior, and the local consistency constraints are sufficient to characterize it.

On graphs with \emph{loops}, the BFE is an approximation: the Bethe factorization does not generally correspond to a valid probability distribution, and BP may not converge.
Nevertheless, when BP does converge, its fixed points remain stationary points of the BFE~\citep{yedidia_constructing_2005,heskes_convexity_2006}.

\begin{remark}
    The Bethe approximation is one member of a broader family.
    Manipulating the local constraints (for instance, imposing full factorization $q(\bm{s}) = \prod_i q_i(s_i)$) recovers the mean-field approximation.
    Other constraint choices yield generalizations such as the Kikuchi free energy~\citep{yedidia_constructing_2005}.
    Different constraint manipulations on the factor graph unify variational message passing, belief propagation, and expectation propagation within a single framework \citep{senoz_message_2022}.
\end{remark}

In models with shared parameters, such as a temporal chain where $\theta$ appears in both the dynamics and observation factors at every time step, the resulting loops make the BFE an approximation rather than an exact decomposition.
Entropic corrections to the BFE yield modified objectives whose stationarity conditions are derived in \refappx{appx:combined_detailed_derivation}.

\section{Detailed Message Passing Derivation for the Combined Objective}
\label{appx:combined_detailed_derivation}

This appendix gives a self-contained $T=1$ derivation of the message-passing scheme for the combined objective. The objective includes both the cross-entropy planning correction and the entropy corrections required for planning-as-inference in active inference. Its message-passing structure is organized around four channels, with $r_{u|x}(u_1|x_0)$ appearing in the numerator of the dynamics kernel.

\subsection{Model, Coordinates, and Constraints}
\label{appx:combined_detailed_setup}

We consider the biased generative model
\begin{equation}\label{eq:comb-detail-generative-model}
    p(y_1, x_1, x_0, \theta, u_1)
    =
    p(\theta)\, p(x_0)\, p(u_1)\, p(x_1 | x_0, \theta, u_1)\, p(y_1 | x_1, \theta)\, \hat{p}_x(x_1)\, \hat{p}_y(y_1).
\end{equation}
The corresponding non-singleton factors are the observation factor $f_{\mathrm{obs}}(y_1,x_1,\theta)=p(y_1|x_1,\theta)$ and the dynamics factor $f_{\mathrm{dyn}}(x_1,x_0,\theta,u_1)=p(x_1|x_0,\theta,u_1)$, together with singleton prior and goal factors on $\theta$, $x_0$, $u_1$, $x_1$, and $y_1$.

The Bethe coordinates consist of the factor beliefs
\begin{equation}
    q_{\mathrm{obs}}(y_1, x_1, \theta), \qquad q_{\mathrm{dyn}}(x_1, x_0, \theta, u_1),
\end{equation}
the singleton beliefs are
\begin{equation}
    q_\theta(\theta), \quad q_{x_0}(x_0), \quad q_{x_1}(x_1), \quad q_{y_1}(y_1), \quad q_{u_1}(u_1),
\end{equation}
and the local-polytope constraints consist of factor normalization,
\begin{subequations}\label{eq:comb-detail-factor-normalization}
    \begin{align}
        \int q_{\mathrm{obs}}(y_1, x_1, \theta)\, \dif y_1\, \dif x_1\, \dif \theta                 & = 1, \\
        \int q_{\mathrm{dyn}}(x_1, x_0, \theta, u_1)\, \dif x_1\, \dif x_0\, \dif \theta\, \dif u_1 & = 1,
    \end{align}
\end{subequations}
and factor-to-singleton consistency,
\begin{subequations}\label{eq:comb-detail-consistency}
    \begin{align}
        \int q_{\mathrm{obs}}(y_1, x_1, \theta)\, \dif x_1\, \dif \theta                 & = q_{y_1}(y_1),     \\
        \int q_{\mathrm{obs}}(y_1, x_1, \theta)\, \dif y_1\, \dif \theta                 & = q_{x_1}(x_1),     \\
        \int q_{\mathrm{obs}}(y_1, x_1, \theta)\, \dif y_1\, \dif x_1                    & = q_\theta(\theta), \\
        \int q_{\mathrm{dyn}}(x_1, x_0, \theta, u_1)\, \dif x_0\, \dif \theta\, \dif u_1 & = q_{x_1}(x_1),     \\
        \int q_{\mathrm{dyn}}(x_1, x_0, \theta, u_1)\, \dif x_1\, \dif \theta\, \dif u_1 & = q_{x_0}(x_0),     \\
        \int q_{\mathrm{dyn}}(x_1, x_0, \theta, u_1)\, \dif x_1\, \dif x_0\, \dif u_1    & = q_\theta(\theta), \\
        \int q_{\mathrm{dyn}}(x_1, x_0, \theta, u_1)\, \dif x_1\, \dif x_0\, \dif \theta & = q_{u_1}(u_1).
    \end{align}
\end{subequations}

The combined objective uses four normalized channels:
\begin{subequations}\label{eq:comb-detail-channels}
    \begin{align}
        r_{y|x\theta}(y_1|x_1,\theta), \qquad
        r_{y|x}(y_1|x_1), \qquad
        r_{x|xu}(x_1|x_0,u_1), \qquad
        r_{u|x}(u_1|x_0),
    \end{align}
\end{subequations}
with normalization constraints
\begin{subequations}\label{eq:comb-detail-channel-norms}
    \begin{align}
        \int r_{y|x\theta}(y_1|x_1,\theta)\, \dif y_1 & = 1 \quad \forall (x_1,\theta), \\
        \int r_{y|x}(y_1|x_1)\, \dif y_1              & = 1 \quad \forall x_1,          \\
        \int r_{x|xu}(x_1|x_0,u_1)\, \dif x_1         & = 1 \quad \forall (x_0,u_1),    \\
        \int r_{u|x}(u_1|x_0)\, \dif u_1              & = 1 \quad \forall x_0.
    \end{align}
\end{subequations}

We also use the derived marginals
\begin{subequations}\label{eq:comb-detail-derived}
    \begin{align}
        q_{\mathrm{sep}}(x_1,\theta)   & := \int q_{\mathrm{obs}}(y_1,x_1,\theta)\, \dif y_1,                   \\
        q_{yx}(y_1,x_1)                & := \int q_{\mathrm{obs}}(y_1,x_1,\theta)\, \dif \theta,                \\
        q_{\mathrm{trip}}(x_1,x_0,u_1) & := \int q_{\mathrm{dyn}}(x_1,x_0,\theta,u_1)\, \dif \theta,            \\
        q_{\mathrm{pair}}(x_0,u_1)     & := \int q_{\mathrm{trip}}(x_1,x_0,u_1)\, \dif x_1,                     \\
        q_{ux}(u_1,x_0)                & := \int q_{\mathrm{dyn}}(x_1,x_0,\theta,u_1)\, \dif x_1\, \dif \theta.
    \end{align}
\end{subequations}
The last quantity is just $q_{ux}(u_1,x_0) = q_{\mathrm{pair}}(x_0,u_1)$, but it is convenient to name it separately when deriving the policy channel.

\subsection{Combined Objective and Modified Kernels}
\label{appx:combined_detailed_objective}

For $T=1$, the correction added to the Bethe free energy is
\begin{equation}\label{eq:comb-detail-correction}
    \Delta F_{\mathrm{comb}}
    =
    \underbrace{
        2\ent{q(y_1|x_1,\theta)}
        - \ent{q(x_1|x_0,u_1)}
        - \ent{q(y_1|x_1)}
    }_{\Delta^{\mathrm{AIF}}}
    +
    \underbrace{
        \ent{q(u_1|x_0)}
    }_{\Delta^{\mathrm{planning}}}.
\end{equation}

Using the variational characterization of conditional entropy, each conditional entropy is reparameterized by a channel. The signs matter:
\begin{itemize}
    \item $+\ent{q(u_1|x_0)}$ contributes $-\Ex{q_{\mathrm{dyn}}}{\log r_{u|x}(u_1|x_0)}$, so $r_{u|x}$ enters in the numerator of the dynamics kernel.
    \item $-\ent{q(x_1|x_0,u_1)}$ contributes $+\Ex{q_{\mathrm{dyn}}}{\log r_{x|xu}(x_1|x_0,u_1)}$, so $r_{x|xu}$ enters in the denominator of the dynamics kernel.
    \item $+2\ent{q(y_1|x_1,\theta)}$ contributes $-2\Ex{q_{\mathrm{obs}}}{\log r_{y|x\theta}(y_1|x_1,\theta)}$, so $r_{y|x\theta}^2$ enters in the numerator of the observation kernel.
    \item $-\ent{q(y_1|x_1)}$ contributes $+\Ex{q_{yx}}{\log r_{y|x}(y_1|x_1)}$, so $r_{y|x}$ enters in the denominator of the observation kernel.
\end{itemize}

The resulting objective is
\begin{equation}\label{eq:comb-detail-full-objective}
    \begin{aligned}
        F_{\mathrm{comb}}[q,r]
         & = \int q_{\mathrm{obs}} \log \frac{q_{\mathrm{obs}}}{p(y_1|x_1,\theta)}\, \dif y_1\, \dif x_1\, \dif \theta
        - 2 \int q_{\mathrm{obs}} \log r_{y|x\theta}\, \dif y_1\, \dif x_1\, \dif \theta                                                    \\
         & \quad + \int q_{yx}(y_1,x_1) \log r_{y|x}(y_1|x_1)\, \dif y_1\, \dif x_1                                                         \\
         & \quad + \int q_{\mathrm{dyn}} \log \frac{q_{\mathrm{dyn}}}{p(x_1|x_0,\theta,u_1)}\, \dif x_1\, \dif x_0\, \dif \theta\, \dif u_1 \\
         & \quad + \int q_{\mathrm{dyn}} \log r_{x|xu}(x_1|x_0,u_1)\, \dif x_1\, \dif x_0\, \dif \theta\, \dif u_1                          \\
         & \quad - \int q_{\mathrm{dyn}} \log r_{u|x}(u_1|x_0)\, \dif x_1\, \dif x_0\, \dif \theta\, \dif u_1                               \\
         & \quad - \int q_{x_0} \log p(x_0)\, \dif x_0
        - \int q_{u_1} \log p(u_1)\, \dif u_1
        - \int q_{y_1} \log \hat{p}_y(y_1)\, \dif y_1                                                                                       \\
         & \quad + (d_\theta - 1)\ent{q_\theta} - \int q_\theta \log p(\theta)\, \dif \theta                                                \\
         & \quad + (d_{x_1} - 1)\ent{q_{x_1}} - \int q_{x_1} \log \hat{p}_x(x_1)\, \dif x_1.
    \end{aligned}
\end{equation}

The variable degrees are
\begin{equation}\label{eq:comb-detail-degrees}
    d_\theta = 3, \qquad d_{x_0} = 2, \qquad d_{x_1} = 3, \qquad d_{y_1} = 2, \qquad d_{u_1} = 2.
\end{equation}

\begin{remark}[Combined kernels]\label{rem:comb-detail-kernels}
    The two non-singleton factors are reweighted by
    \begin{subequations}\label{eq:comb-detail-kernels}
        \begin{align}
            \tilde{f}_{\mathrm{obs}}(y_1,x_1,\theta)
             & = \frac{p(y_1|x_1,\theta)\, r_{y|x\theta}^2(y_1|x_1,\theta)}{r_{y|x}(y_1|x_1)}, \label{eq:comb-detail-obs-kernel} \\
            \tilde{f}_{\mathrm{dyn}}^{\mathrm{comb}}(x_1,x_0,\theta,u_1)
             & = \frac{p(x_1|x_0,\theta,u_1)\, r_{u|x}(u_1|x_0)}{r_{x|xu}(x_1|x_0,u_1)}. \label{eq:comb-detail-dyn-kernel}
        \end{align}
    \end{subequations}
    The observation kernel is
    \begin{equation}
        \tilde{f}_{\mathrm{obs}}(y_1,x_1,\theta)
        =
        \frac{p(y_1|x_1,\theta)\, r_{y|x\theta}^2(y_1|x_1,\theta)}{r_{y|x}(y_1|x_1)},
    \end{equation}
    while the dynamics kernel is a genuine ratio in which $r_{u|x}$ sharpens action selection and $r_{x|xu}$ spreads mass over predicted futures.
\end{remark}

\subsection{Lagrangian}
\label{appx:combined_detailed_lagrangian}

We introduce Lagrange multipliers for all normalization and marginalization constraints. The factor and consistency multipliers are
\begin{equation}
    \lambda_{\mathrm{obs}}, \quad
    \lambda_{\mathrm{dyn}}, \quad
    \lambda_{y_1}(y_1), \quad
    \lambda_\theta^{(\mathrm{obs})}(\theta), \quad
    \lambda_{x_1}^{(\mathrm{obs})}(x_1), \quad
    \lambda_{x_1}^{(\mathrm{dyn})}(x_1), \quad
    \lambda_{x_0}(x_0), \quad
    \lambda_{u_1}(u_1), \quad
    \lambda_\theta^{(\mathrm{dyn})}(\theta),
\end{equation}
and the channel multipliers are
\begin{equation}
    \nu_{\mathrm{obs}}(x_1,\theta), \qquad \nu_{y|x}(x_1), \qquad \nu_x(x_0,u_1), \qquad \nu_{u|x}(x_0).
\end{equation}

The full Lagrangian is
\begin{equation}\label{eq:comb-detail-lagrangian}
    \begin{aligned}
        \mathcal{L}_{\mathrm{comb}}
         & = F_{\mathrm{comb}}[q,r]                                                                                                                                                       \\
         & \quad + \lambda_{\mathrm{obs}} \left(\int q_{\mathrm{obs}}(y_1, x_1, \theta)\, \dif y_1\, \dif x_1\, \dif \theta - 1\right)                                                    \\
         & \quad + \lambda_{\mathrm{dyn}} \left(\int q_{\mathrm{dyn}}(x_1, x_0, \theta, u_1)\, \dif x_1\, \dif x_0\, \dif \theta\, \dif u_1 - 1\right)                                    \\
         & \quad + \int \lambda_{y_1}(y_1) \left(\int q_{\mathrm{obs}}(y_1, x_1, \theta)\, \dif x_1\, \dif \theta - q_{y_1}(y_1)\right) \dif y_1                                          \\
         & \quad + \int \lambda_\theta^{(\mathrm{obs})}(\theta) \left(\int q_{\mathrm{obs}}(y_1, x_1, \theta)\, \dif y_1\, \dif x_1 - q_\theta(\theta)\right) \dif \theta                 \\
         & \quad + \int \lambda_{x_1}^{(\mathrm{obs})}(x_1) \left(\int q_{\mathrm{obs}}(y_1, x_1, \theta)\, \dif y_1\, \dif \theta - q_{x_1}(x_1)\right) \dif x_1                         \\
         & \quad + \int \lambda_{x_1}^{(\mathrm{dyn})}(x_1) \left(\int q_{\mathrm{dyn}}(x_1, x_0, \theta, u_1)\, \dif x_0\, \dif \theta\, \dif u_1 - q_{x_1}(x_1)\right) \dif x_1         \\
         & \quad + \int \lambda_{x_0}(x_0) \left(\int q_{\mathrm{dyn}}(x_1, x_0, \theta, u_1)\, \dif x_1\, \dif \theta\, \dif u_1 - q_{x_0}(x_0)\right) \dif x_0                          \\
         & \quad + \int \lambda_{u_1}(u_1) \left(\int q_{\mathrm{dyn}}(x_1, x_0, \theta, u_1)\, \dif x_1\, \dif x_0\, \dif \theta - q_{u_1}(u_1)\right) \dif u_1                          \\
         & \quad + \int \lambda_\theta^{(\mathrm{dyn})}(\theta) \left(\int q_{\mathrm{dyn}}(x_1, x_0, \theta, u_1)\, \dif x_1\, \dif x_0\, \dif u_1 - q_\theta(\theta)\right) \dif \theta \\
         & \quad + \iint \nu_{\mathrm{obs}}(x_1, \theta) \left(\int r_{y|x\theta}(y_1|x_1,\theta)\, \dif y_1 - 1\right) \dif x_1\, \dif \theta                                            \\
         & \quad + \iint \nu_{y|x}(x_1) \left(\int r_{y|x}(y_1|x_1)\, \dif y_1 - 1\right) \dif x_1                                                                                        \\
         & \quad + \iint \nu_x(x_0, u_1) \left(\int r_{x|xu}(x_1|x_0,u_1)\, \dif x_1 - 1\right) \dif x_0\, \dif u_1                                                                       \\
         & \quad + \int \nu_{u|x}(x_0) \left(\int r_{u|x}(u_1|x_0)\, \dif u_1 - 1\right) \dif x_0.
    \end{aligned}
\end{equation}

\subsection{Stationarity Conditions}
\label{appx:combined_detailed_stationarity}

\subsubsection{Variation with respect to \texorpdfstring{$q_{\mathrm{obs}}$}{qobs}}

The stationarity equation for the observation factor is
\begin{proposition}[Observation factor belief for the combined objective]\label{prop:comb-detail-qobs}
    At stationarity,
    \begin{equation}\label{eq:comb-detail-qobs}
        q_{\mathrm{obs}}^*(y_1,x_1,\theta)
        \propto
        \frac{p(y_1|x_1,\theta)\, r_{y|x\theta}^2(y_1|x_1,\theta)}{r_{y|x}(y_1|x_1)}
        e^{-\lambda_{y_1}(y_1)}
        e^{-\lambda_\theta^{(\mathrm{obs})}(\theta)}
        e^{-\lambda_{x_1}^{(\mathrm{obs})}(x_1)}.
    \end{equation}
\end{proposition}

\subsubsection{Variation with respect to \texorpdfstring{$q_{\mathrm{dyn}}$}{qdyn}}

The dynamics-side terms now contain both channel contributions:
\begin{equation}
    \begin{aligned}
         & \int q_{\mathrm{dyn}} \log q_{\mathrm{dyn}}
        - \int q_{\mathrm{dyn}} \log p(x_1|x_0,\theta,u_1)
        + \int q_{\mathrm{dyn}} \log r_{x|xu}(x_1|x_0,u_1)     \\
         & \quad - \int q_{\mathrm{dyn}} \log r_{u|x}(u_1|x_0)
        + \text{(multiplier terms)}.
    \end{aligned}
\end{equation}
Taking $\frac{\delta \mathcal{L}_{\mathrm{comb}}}{\delta q_{\mathrm{dyn}}}=0$ gives
\begin{equation}
    \log q_{\mathrm{dyn}} + 1 - \log p(x_1|x_0,\theta,u_1) + \log r_{x|xu}(x_1|x_0,u_1) - \log r_{u|x}(u_1|x_0)
    + \lambda_{\mathrm{dyn}} + \lambda_{x_1}^{(\mathrm{dyn})} + \lambda_{x_0} + \lambda_{u_1} + \lambda_\theta^{(\mathrm{dyn})} = 0.
\end{equation}

\begin{proposition}[Combined dynamics factor belief]\label{prop:comb-detail-qdyn}
    At stationarity,
    \begin{equation}\label{eq:comb-detail-qdyn}
        q_{\mathrm{dyn}}^*(x_1,x_0,\theta,u_1)
        \propto
        \frac{p(x_1|x_0,\theta,u_1)\, r_{u|x}(u_1|x_0)}{r_{x|xu}(x_1|x_0,u_1)}
        e^{-\lambda_{x_1}^{(\mathrm{dyn})}(x_1)}
        e^{-\lambda_{x_0}(x_0)}
        e^{-\lambda_{u_1}(u_1)}
        e^{-\lambda_\theta^{(\mathrm{dyn})}(\theta)}.
    \end{equation}
    The dynamics factor is therefore reweighted by a policy-dependent numerator together with a predictive denominator.
\end{proposition}

\subsubsection{Variation with respect to the observation channels}

The observation-side channels satisfy
\begin{subequations}\label{eq:comb-detail-obs-channel-stationary}
    \begin{align}
        r_{y|x\theta}^*(y_1|x_1,\theta) & = \frac{q_{\mathrm{obs}}(y_1,x_1,\theta)}{q_{\mathrm{sep}}(x_1,\theta)} = q(y_1|x_1,\theta), \label{eq:comb-detail-r-yxt} \\
        r_{y|x}^*(y_1|x_1)              & = \frac{q_{yx}(y_1,x_1)}{q_{x_1}(x_1)} = q(y_1|x_1). \label{eq:comb-detail-r-yx}
    \end{align}
\end{subequations}

\subsubsection{Variation with respect to \texorpdfstring{$r_{x|xu}$}{r(x|xu)}}

The predictive dynamics channel is obtained from
\begin{equation}
    \frac{\delta \mathcal{L}_{\mathrm{comb}}}{\delta r_{x|xu}(x_1|x_0,u_1)}
    =
    \frac{q_{\mathrm{trip}}(x_1,x_0,u_1)}{r_{x|xu}(x_1|x_0,u_1)} + \nu_x(x_0,u_1) = 0.
\end{equation}

\begin{proposition}[Predictive dynamics channel]\label{prop:comb-detail-rx}
    At stationarity,
    \begin{equation}\label{eq:comb-detail-rx}
        r_{x|xu}^*(x_1|x_0,u_1)
        =
        \frac{q_{\mathrm{trip}}(x_1,x_0,u_1)}{q_{\mathrm{pair}}(x_0,u_1)}
        =
        q(x_1|x_0,u_1).
    \end{equation}
\end{proposition}

\subsubsection{Variation with respect to \texorpdfstring{$r_{u|x}$}{r(u|x)}}

Since $r_{u|x}$ does not depend on $x_1$ or $\theta$, only the marginal $q_{ux}(u_1,x_0)$ matters:
\begin{equation}
    - \int q_{ux}(u_1,x_0)\log r_{u|x}(u_1|x_0)\, \dif u_1\, \dif x_0
    + \int \nu_{u|x}(x_0)\left(\int r_{u|x}(u_1|x_0)\, \dif u_1 - 1\right)\dif x_0.
\end{equation}

Taking a pointwise derivative yields
\begin{equation}
    -\frac{q_{ux}(u_1,x_0)}{r_{u|x}(u_1|x_0)} + \nu_{u|x}(x_0) = 0.
\end{equation}
Normalization implies $\nu_{u|x}(x_0)=q_{x_0}(x_0)$.

\begin{proposition}[Policy channel]\label{prop:comb-detail-rux}
    At stationarity,
    \begin{equation}\label{eq:comb-detail-rux}
        r_{u|x}^*(u_1|x_0)
        =
        \frac{q_{ux}(u_1,x_0)}{q_{x_0}(x_0)}
        =
        q(u_1|x_0).
    \end{equation}
\end{proposition}

\subsubsection{Variation with respect to singleton beliefs}

The singleton calculations are elementary because they depend only on degree counting. Since
\begin{equation}
    d_{y_1}=d_{x_0}=d_{u_1}=2, \qquad d_{\theta}=d_{x_1}=3,
\end{equation}
the degree-2 and degree-3 variables behave differently. In particular:
\begin{itemize}
    \item for degree-2 variables $y_1$, $x_0$, and $u_1$, the KL term and the singleton entropy cancel, so
          \begin{equation}\label{eq:comb-detail-deg2}
              \lambda_{y_1}(y_1) = -\log \hat{p}_y(y_1), \qquad
              \lambda_{x_0}(x_0) = -\log p(x_0), \qquad
              \lambda_{u_1}(u_1) = -\log p(u_1);
          \end{equation}
    \item for degree-3 variables $\theta$ and $x_1$, one obtains only constraints on sums of multipliers,
          \begin{equation}\label{eq:comb-detail-deg3}
              \lambda_\theta^{(\mathrm{obs})}(\theta) + \lambda_\theta^{(\mathrm{dyn})}(\theta)
              = -\log q_\theta^*(\theta) - 1 - \log p(\theta),
          \end{equation}
          \begin{equation}\label{eq:comb-detail-deg3-x1}
              \lambda_{x_1}^{(\mathrm{obs})}(x_1) + \lambda_{x_1}^{(\mathrm{dyn})}(x_1)
              = -\log q_{x_1}^*(x_1) - 1 - \log \hat{p}_x(x_1).
          \end{equation}
\end{itemize}
The extra policy channel does not alter these facts, because it modifies only the non-singleton dynamics factor and does not introduce any new local-consistency constraint involving singletons.

\subsection{Solving the Stationarity Equations as Message Passing}
\label{appx:combined_detailed_messages}

Once the kernels \eqref{eq:comb-detail-kernels} are identified, the rest of the derivation follows the standard Bethe logic. For a factor $a$ with kernel $\tilde{f}_a$ and neighboring variables $\mathcal{E}(a)$, the factor-to-variable update is
\begin{equation}\label{eq:comb-detail-factor-to-var}
    \mu_{a\to j}(s_j)
    \propto
    \int \tilde{f}_a(\bm{s}_a)
    \prod_{i \in \mathcal{E}(a)\setminus j} \mu_{i\to a}(s_i)\,
    \dif \bm{s}_{a\setminus j}.
\end{equation}

\subsubsection{Observation-factor messages}

Using the observation kernel in \eqref{eq:comb-detail-obs-kernel}, the observation-factor messages are
\begin{subequations}\label{eq:comb-detail-msgs-obs}
    \begin{align}
        \mu_{\mathrm{obs}\to\theta}(\theta)
         & =
        \iint \frac{p(y_1|x_1,\theta)\, r_{y|x\theta}^2(y_1|x_1,\theta)}{r_{y|x}(y_1|x_1)}
        \mu_{y_1\to\mathrm{obs}}(y_1)\mu_{x_1\to\mathrm{obs}}(x_1)\,
        \dif y_1\, \dif x_1,    \\
        \mu_{\mathrm{obs}\to x_1}(x_1)
         & =
        \iint \frac{p(y_1|x_1,\theta)\, r_{y|x\theta}^2(y_1|x_1,\theta)}{r_{y|x}(y_1|x_1)}
        \mu_{y_1\to\mathrm{obs}}(y_1)\mu_{\theta\to\mathrm{obs}}(\theta)\,
        \dif y_1\, \dif \theta, \\
        \mu_{\mathrm{obs}\to y_1}(y_1)
         & =
        \iint \frac{p(y_1|x_1,\theta)\, r_{y|x\theta}^2(y_1|x_1,\theta)}{r_{y|x}(y_1|x_1)}
        \mu_{x_1\to\mathrm{obs}}(x_1)\mu_{\theta\to\mathrm{obs}}(\theta)\,
        \dif x_1\, \dif \theta.
    \end{align}
\end{subequations}

\subsubsection{Dynamics-factor messages}

The combined dynamics kernel yields
\begin{subequations}\label{eq:comb-detail-msgs-dyn}
    \begin{align}
        \mu_{\mathrm{dyn}\to\theta}(\theta)
         & =
        \iiint \frac{p(x_1|x_0,\theta,u_1)\, r_{u|x}(u_1|x_0)}{r_{x|xu}(x_1|x_0,u_1)}
        \mu_{x_1\to\mathrm{dyn}}(x_1)\mu_{x_0\to\mathrm{dyn}}(x_0)\mu_{u_1\to\mathrm{dyn}}(u_1)\,
        \dif x_1\, \dif x_0\, \dif u_1,    \\
        \mu_{\mathrm{dyn}\to x_1}(x_1)
         & =
        \iiint \frac{p(x_1|x_0,\theta,u_1)\, r_{u|x}(u_1|x_0)}{r_{x|xu}(x_1|x_0,u_1)}
        \mu_{x_0\to\mathrm{dyn}}(x_0)\mu_{u_1\to\mathrm{dyn}}(u_1)\mu_{\theta\to\mathrm{dyn}}(\theta)\,
        \dif x_0\, \dif u_1\, \dif \theta, \\
        \mu_{\mathrm{dyn}\to x_0}(x_0)
         & =
        \iiint \frac{p(x_1|x_0,\theta,u_1)\, r_{u|x}(u_1|x_0)}{r_{x|xu}(x_1|x_0,u_1)}
        \mu_{x_1\to\mathrm{dyn}}(x_1)\mu_{u_1\to\mathrm{dyn}}(u_1)\mu_{\theta\to\mathrm{dyn}}(\theta)\,
        \dif x_1\, \dif u_1\, \dif \theta, \\
        \mu_{\mathrm{dyn}\to u_1}(u_1)
         & =
        \iiint \frac{p(x_1|x_0,\theta,u_1)\, r_{u|x}(u_1|x_0)}{r_{x|xu}(x_1|x_0,u_1)}
        \mu_{x_1\to\mathrm{dyn}}(x_1)\mu_{x_0\to\mathrm{dyn}}(x_0)\mu_{\theta\to\mathrm{dyn}}(\theta)\,
        \dif x_1\, \dif x_0\, \dif \theta.
    \end{align}
\end{subequations}

\subsubsection{Factor beliefs and singleton beliefs}

The factor beliefs are kernel times incoming messages:
\begin{subequations}\label{eq:comb-detail-factor-beliefs}
    \begin{align}
        q_{\mathrm{obs}}^*(y_1,x_1,\theta)
         & \propto \tilde{f}_{\mathrm{obs}}(y_1,x_1,\theta)\,
        \mu_{y_1\to\mathrm{obs}}(y_1)\mu_{x_1\to\mathrm{obs}}(x_1)\mu_{\theta\to\mathrm{obs}}(\theta), \\
        q_{\mathrm{dyn}}^*(x_1,x_0,\theta,u_1)
         & \propto \tilde{f}_{\mathrm{dyn}}^{\mathrm{comb}}(x_1,x_0,\theta,u_1)\,
        \mu_{x_1\to\mathrm{dyn}}(x_1)\mu_{x_0\to\mathrm{dyn}}(x_0)\mu_{\theta\to\mathrm{dyn}}(\theta)\mu_{u_1\to\mathrm{dyn}}(u_1).
    \end{align}
\end{subequations}

Singleton beliefs keep the usual sum-product form:
\begin{subequations}\label{eq:comb-detail-singletons}
    \begin{align}
        q_\theta^*(\theta) & \propto p(\theta)\, \mu_{\mathrm{obs}\to\theta}(\theta)\, \mu_{\mathrm{dyn}\to\theta}(\theta), \\
        q_{x_1}^*(x_1)     & \propto \hat{p}_x(x_1)\, \mu_{\mathrm{obs}\to x_1}(x_1)\, \mu_{\mathrm{dyn}\to x_1}(x_1),      \\
        q_{x_0}^*(x_0)     & \propto p(x_0)\, \mu_{\mathrm{dyn}\to x_0}(x_0),                                               \\
        q_{u_1}^*(u_1)     & \propto p(u_1)\, \mu_{\mathrm{dyn}\to u_1}(u_1),                                               \\
        q_{y_1}^*(y_1)     & \propto \hat{p}_y(y_1)\, \mu_{\mathrm{obs}\to y_1}(y_1).
    \end{align}
\end{subequations}

\subsubsection{Channel updates}

At a fixed point, all four channels recover the corresponding conditionals under the current factor beliefs:
\begin{subequations}\label{eq:comb-detail-channel-updates}
    \begin{align}
        r_{y|x\theta}^*(y_1|x_1,\theta) & = q(y_1|x_1,\theta), \\
        r_{y|x}^*(y_1|x_1)              & = q(y_1|x_1),        \\
        r_{x|xu}^*(x_1|x_0,u_1)         & = q(x_1|x_0,u_1),    \\
        r_{u|x}^*(u_1|x_0)              & = q(u_1|x_0).
    \end{align}
\end{subequations}

\begin{remark}[Dynamics-side min-max structure]
    The two dynamics channels act on different conditionals and with opposite signs. The denominator $r_{x|xu}$ maximizes $\ent{q(x_1|x_0,u_1)}$, while the numerator $r_{u|x}$ minimizes $\ent{q(u_1|x_0)}$. The combined dynamics factor therefore interpolates between commitment over actions and spreading over future states.
\end{remark}

\subsection{Generic scheme for arbitrary \texorpdfstring{$T$}{T}}
\label{appx:combined_detailed_generic}

The passage from $T=1$ to arbitrary horizons is immediate because all entropy corrections are additive across time, so each time step receives its own local channels
\begin{equation}
    r_{y|x\theta,t}(y_t|x_t,\theta), \quad
    r_{y|x,t}(y_t|x_t), \quad
    r_{x|xu,t}(x_t|x_{t-1},u_t), \quad
    r_{u|x,t}(u_t|x_{t-1}).
\end{equation}

The per-time-step kernels are
\begin{subequations}\label{eq:comb-detail-generic-kernels}
    \begin{align}
        \tilde{f}_{\mathrm{obs}_t}(y_t,x_t,\theta)
         & = \frac{p(y_t|x_t,\theta)\, r_{y|x\theta,t}^2(y_t|x_t,\theta)}{r_{y|x,t}(y_t|x_t)},       \\
        \tilde{f}_{\mathrm{dyn}_t}^{\mathrm{comb}}(x_t,x_{t-1},\theta,u_t)
         & = \frac{p(x_t|x_{t-1},\theta,u_t)\, r_{u|x,t}(u_t|x_{t-1})}{r_{x|xu,t}(x_t|x_{t-1},u_t)}.
    \end{align}
\end{subequations}

The full multi-step scheme is obtained by applying the usual sum-product updates with the kernels in \eqref{eq:comb-detail-generic-kernels}.

\subsubsection{Messages from Observation Factor \texorpdfstring{$f_{\mathrm{obs}_t}$}{fobs,t}}

\begin{subequations}\label{eq:comb-detail-gen-msgs-obs}
    \begin{align}
        \mu_{\mathrm{obs}_t\to\theta}(\theta)
         & =
        \iint \frac{p(y_t|x_t,\theta)\, r_{y|x\theta,t}^2(y_t|x_t,\theta)}{r_{y|x,t}(y_t|x_t)}
        \mu_{y_t\to\mathrm{obs}_t}(y_t)\mu_{x_t\to\mathrm{obs}_t}(x_t)\,
        \dif y_t\, \dif x_t,    \\
        \mu_{\mathrm{obs}_t\to x_t}(x_t)
         & =
        \iint \frac{p(y_t|x_t,\theta)\, r_{y|x\theta,t}^2(y_t|x_t,\theta)}{r_{y|x,t}(y_t|x_t)}
        \mu_{y_t\to\mathrm{obs}_t}(y_t)\mu_{\theta\to\mathrm{obs}_t}(\theta)\,
        \dif y_t\, \dif \theta, \\
        \mu_{\mathrm{obs}_t\to y_t}(y_t)
         & =
        \iint \frac{p(y_t|x_t,\theta)\, r_{y|x\theta,t}^2(y_t|x_t,\theta)}{r_{y|x,t}(y_t|x_t)}
        \mu_{x_t\to\mathrm{obs}_t}(x_t)\mu_{\theta\to\mathrm{obs}_t}(\theta)\,
        \dif x_t\, \dif \theta.
    \end{align}
\end{subequations}

\subsubsection{Messages from Dynamics Factor \texorpdfstring{$f_{\mathrm{dyn}_t}$}{fdyn,t}}

\begin{subequations}\label{eq:comb-detail-gen-msgs-dyn}
    \begin{align}
        \mu_{\mathrm{dyn}_t\to\theta}(\theta)
         & =
        \iiint \frac{p(x_t|x_{t-1},\theta,u_t)\, r_{u|x,t}(u_t|x_{t-1})}{r_{x|xu,t}(x_t|x_{t-1},u_t)} \notag \\
         & \qquad \times
        \mu_{x_t\to\mathrm{dyn}_t}(x_t)\mu_{x_{t-1}\to\mathrm{dyn}_t}(x_{t-1})\mu_{u_t\to\mathrm{dyn}_t}(u_t)\,
        \dif x_t\, \dif x_{t-1}\, \dif u_t,    \\
        \mu_{\mathrm{dyn}_t\to x_t}(x_t)
         & =
        \iiint \frac{p(x_t|x_{t-1},\theta,u_t)\, r_{u|x,t}(u_t|x_{t-1})}{r_{x|xu,t}(x_t|x_{t-1},u_t)}
        \mu_{x_{t-1}\to\mathrm{dyn}_t}(x_{t-1})\mu_{u_t\to\mathrm{dyn}_t}(u_t)\mu_{\theta\to\mathrm{dyn}_t}(\theta)\,
        \dif x_{t-1}\, \dif u_t\, \dif \theta, \\
        \mu_{\mathrm{dyn}_t\to x_{t-1}}(x_{t-1})
         & =
        \iiint \frac{p(x_t|x_{t-1},\theta,u_t)\, r_{u|x,t}(u_t|x_{t-1})}{r_{x|xu,t}(x_t|x_{t-1},u_t)}
        \mu_{x_t\to\mathrm{dyn}_t}(x_t)\mu_{u_t\to\mathrm{dyn}_t}(u_t)\mu_{\theta\to\mathrm{dyn}_t}(\theta)\,
        \dif x_t\, \dif u_t\, \dif \theta,     \\
        \mu_{\mathrm{dyn}_t\to u_t}(u_t)
         & =
        \iiint \frac{p(x_t|x_{t-1},\theta,u_t)\, r_{u|x,t}(u_t|x_{t-1})}{r_{x|xu,t}(x_t|x_{t-1},u_t)}
        \mu_{x_t\to\mathrm{dyn}_t}(x_t)\mu_{x_{t-1}\to\mathrm{dyn}_t}(x_{t-1})\mu_{\theta\to\mathrm{dyn}_t}(\theta)\,
        \dif x_t\, \dif x_{t-1}\, \dif \theta.
    \end{align}
\end{subequations}

\subsubsection{Factor Beliefs}

\begin{subequations}\label{eq:comb-detail-gen-factor-beliefs}
    \begin{align}
        q_{\mathrm{obs},t}^*(y_t,x_t,\theta)
         & \propto \tilde{f}_{\mathrm{obs}_t}(y_t,x_t,\theta)\,
        \mu_{y_t\to\mathrm{obs}_t}(y_t)\mu_{x_t\to\mathrm{obs}_t}(x_t)\mu_{\theta\to\mathrm{obs}_t}(\theta), \\
        q_{\mathrm{dyn},t}^*(x_t,x_{t-1},\theta,u_t)
         & \propto \tilde{f}_{\mathrm{dyn}_t}^{\mathrm{comb}}(x_t,x_{t-1},\theta,u_t)\,
        \mu_{x_t\to\mathrm{dyn}_t}(x_t)\mu_{x_{t-1}\to\mathrm{dyn}_t}(x_{t-1})\mu_{\theta\to\mathrm{dyn}_t}(\theta)\mu_{u_t\to\mathrm{dyn}_t}(u_t).
    \end{align}
\end{subequations}

\subsubsection{Channel Updates}

At a fixed point, each time step has the four local channel updates
\begin{subequations}\label{eq:comb-detail-gen-channel-updates}
    \begin{align}
        r_{y|x\theta,t}^*(y_t|x_t,\theta) & = q_t(y_t|x_t,\theta),  \\
        r_{y|x,t}^*(y_t|x_t)              & = q_t(y_t|x_t),         \\
        r_{x|xu,t}^*(x_t|x_{t-1},u_t)     & = q_t(x_t|x_{t-1},u_t), \\
        r_{u|x,t}^*(u_t|x_{t-1})          & = q_t(u_t|x_{t-1}).
    \end{align}
\end{subequations}

\subsubsection{Singleton Beliefs}

\begin{subequations}\label{eq:comb-detail-gen-singletons}
    \begin{align}
        q_{x_t}^*(x_t)
         & \propto \hat{p}_x(x_t)\, \mu_{\mathrm{obs}_t\to x_t}(x_t)\, \mu_{\mathrm{dyn}_t\to x_t}(x_t)\, \mu_{\mathrm{dyn}_{t+1}\to x_t}(x_t), \\
        q_\theta^*(\theta)
         & \propto p(\theta)\prod_{\tau=1}^{T} \mu_{\mathrm{obs}_\tau\to\theta}(\theta)\, \mu_{\mathrm{dyn}_\tau\to\theta}(\theta),             \\
        q_{u_t}^*(u_t)
         & \propto p(u_t)\, \mu_{\mathrm{dyn}_t\to u_t}(u_t),                                                                                   \\
        q_{y_t}^*(y_t)
         & \propto \hat{p}_y(y_t)\, \mu_{\mathrm{obs}_t\to y_t}(y_t).
    \end{align}
\end{subequations}

The boundary conditions follow the usual temporal-edge conventions: at $t=1$, $\mu_{x_0\to\mathrm{dyn}_1}(x_0)=p(x_0)$; at $t=T$, the message $\mu_{\mathrm{dyn}_{T+1}\to x_T}$ is absent.

\subsection{Interpretation}
\label{appx:combined_detailed_interpretation}

At a fixed point,
\begin{equation}\label{eq:comb-detail-fixed-point-kernel}
    \tilde{f}_{\mathrm{dyn}_t}^{\mathrm{comb}}(x_t,x_{t-1},\theta,u_t)
    =
    \frac{p(x_t|x_{t-1},\theta,u_t)\, q(u_t|x_{t-1})}{q(x_t|x_{t-1},u_t)}.
\end{equation}
This exposes the objective transparently: the numerator implements policy commitment, while the denominator implements the risk-sensitive correction over predicted futures. The observation factor simultaneously balances the two observation-side entropy terms through $r_{y|x\theta}$ and $r_{y|x}$.

\section{Message Passing Derivation for VBP}
\label{appx:vbp_derivation}

This appendix derives the message-passing equations for the Value Belief Propagation (VBP) scheme that implements cross-entropy planning (\Cref{subsec:cross_entropy_planning}). The factor graph and Bethe approximation are identical to the combined-objective derivation (\refappx{appx:combined_detailed_derivation}); only the entropy correction differs. VBP requires a single channel reparameterization, making the derivation substantially simpler.

\subsection{Coordinate System}
\label{appx:vbp_coordinates}

We use the same generative model, factor graph, Bethe coordinates, and local-polytope constraints as in \refappx{appx:combined_detailed_derivation}: factor beliefs $q_{\mathrm{obs}}(y_1, x_1, \theta)$ and $q_{\mathrm{dyn}}(x_1, x_0, \theta, u_1)$, plus singleton beliefs $q_\theta, q_{x_0}, q_{x_1}, q_{y_1}, q_{u_1}$.

The single difference is the channel set. VBP introduces one channel:
\begin{equation}\label{eq:vbp-channel}
    r_{u|x}(u_1 | x_0), \quad \text{subject to} \quad \int r_{u|x}(u_1 | x_0)\, \dif u_1 = 1 \quad \forall\, x_0.
\end{equation}
This channel parameterizes the conditional entropy of actions given states via:
\begin{equation}\label{eq:vbp-channel-identity}
    \ent{q(u_1 | x_0)} = \min_{r_{u|x}} \Ex{q_{\mathrm{pair}}(x_0, u_1)}{-\log r_{u|x}(u_1 | x_0)},
\end{equation}
where the minimum is attained at $r_{u|x}(u_1 | x_0) = q(u_1 | x_0)$, and
\begin{equation}\label{eq:vbp-q-pair}
    q_{\mathrm{pair}}(x_0, u_1) := \iint q_{\mathrm{dyn}}(x_1, x_0, \theta, u_1)\, \dif x_1\, \dif \theta
\end{equation}
is the marginal of the dynamics factor belief over $(x_0, u_1)$.

\subsection{Objective Function}
\label{appx:vbp_objective}

The VBP objective adds the cross-entropy planning correction~\eqref{eq:cross_entropy_correction} to the usual Bethe free energy:
\begin{equation}\label{eq:vbp-correction}
    \Delta F_{\mathrm{VBP}} = +\ent{q(u_1 | x_0)}.
\end{equation}

After channel reparameterization via~\eqref{eq:vbp-channel-identity}:
\begin{equation}\label{eq:vbp-full-objective}
    \begin{aligned}
        F_{\mathrm{VBP}}[q, r_{u|x}] & = \int q_{\mathrm{obs}} \log \frac{q_{\mathrm{obs}}}{p(y_1 | x_1, \theta)}\, \dif y_1\, \dif x_1\, \dif \theta                       \\
                                     & \quad + \int q_{\mathrm{dyn}} \log \frac{q_{\mathrm{dyn}}}{p(x_1 | x_0, \theta, u_1)}\, \dif x_1\, \dif x_0\, \dif \theta\, \dif u_1 \\
                                     & \quad - \int q_{\mathrm{dyn}} \log r_{u|x}(u_1 | x_0)\, \dif x_1\, \dif x_0\, \dif \theta\, \dif u_1                                 \\
                                     & \quad + \text{(the usual singleton Bethe terms)}.
    \end{aligned}
\end{equation}

\begin{remark}
    Since $+\ent{q(u_1|x_0)}$ carries a positive sign, the channel reparameterization contributes $-\Ex{q}{\log r_{u|x}}$ to the objective. Consequently, $r_{u|x}$ appears in the \emph{numerator} of the dynamics kernel (as a multiplicative factor), in contrast to the AIF dynamics channel $r_{x|xu}$ which appears in the denominator~\eqref{eq:modified_dyn}.
\end{remark}

\subsection{Stationarity Conditions}
\label{appx:vbp_stationarity}

We form the Lagrangian with the same normalization and consistency multipliers as in \refappx{appx:combined_detailed_lagrangian}, except that the three observation and predictive-dynamics channel multipliers are replaced by a single multiplier $\nu_{u|x}(x_0)$ for the channel normalization constraint~\eqref{eq:vbp-channel}.

\subsubsection{Variation with respect to \texorpdfstring{$q_{\mathrm{obs}}$}{qobs}}

The observation factor receives no channel correction. The stationarity condition is identical to standard Bethe:
\begin{equation}\label{eq:vbp-qobs-stationary}
    q_{\mathrm{obs}}^*(y_1, x_1, \theta) \propto p(y_1 | x_1, \theta)\, e^{-\lambda_{y_1}(y_1)}\, e^{-\lambda_\theta^{(\mathrm{obs})}(\theta)}\, e^{-\lambda_{x_1}^{(\mathrm{obs})}(x_1)}.
\end{equation}
The observation kernel is simply $p(y_1 | x_1, \theta)$, as in standard belief propagation.

\subsubsection{Variation with respect to \texorpdfstring{$q_{\mathrm{dyn}}$}{qdyn}}

The $q_{\mathrm{dyn}}$-dependent terms include the channel correction $-\int q_{\mathrm{dyn}} \log r_{u|x}(u_1 | x_0)$. Taking $\frac{\delta \mathcal{L}}{\delta q_{\mathrm{dyn}}} = 0$:
\begin{equation}\label{eq:vbp-qdyn-stationarity-eq}
    \log q_{\mathrm{dyn}} + 1 - \log p(x_1 | x_0, \theta, u_1) - \log r_{u|x}(u_1 | x_0) + \lambda_{\mathrm{dyn}} + \lambda_{x_1}^{(\mathrm{dyn})} + \lambda_{x_0} + \lambda_{u_1} + \lambda_\theta^{(\mathrm{dyn})} = 0.
\end{equation}

\begin{proposition}[VBP dynamics factor belief]\label{prop:vbp-qdyn}
    At stationarity:
    \begin{equation}\label{eq:vbp-qdyn-stationary}
        q_{\mathrm{dyn}}^*(x_1, x_0, \theta, u_1) \propto p(x_1 | x_0, \theta, u_1)\, r_{u|x}(u_1 | x_0)\, e^{-\lambda_{x_1}^{(\mathrm{dyn})}(x_1)}\, e^{-\lambda_{x_0}(x_0)}\, e^{-\lambda_{u_1}(u_1)}\, e^{-\lambda_\theta^{(\mathrm{dyn})}(\theta)}.
    \end{equation}
    The product $\tilde{f}_{\mathrm{dyn}}(x_1, x_0, \theta, u_1) := p(x_1 | x_0, \theta, u_1)\, r_{u|x}(u_1 | x_0)$ is the VBP dynamics kernel.
\end{proposition}

\subsubsection{Variation with respect to \texorpdfstring{$r_{u|x}$}{r(u|x)}}

The $r_{u|x}$-dependent terms in the Lagrangian are:
\begin{equation}
    -\int q_{\mathrm{pair}}(x_0, u_1) \log r_{u|x}(u_1 | x_0)\, \dif u_1\, \dif x_0 + \int \nu_{u|x}(x_0) \left(\int r_{u|x}(u_1 | x_0)\, \dif u_1 - 1\right) \dif x_0.
\end{equation}

Pointwise derivative:
\begin{equation}
    -\frac{q_{\mathrm{pair}}(x_0, u_1)}{r_{u|x}(u_1 | x_0)} + \nu_{u|x}(x_0) = 0.
\end{equation}

Imposing normalization $\int r_{u|x}\, \dif u_1 = 1$ yields $\nu_{u|x}(x_0) = q_{x_0}(x_0)$, where $q_{x_0}(x_0) = \int q_{\mathrm{pair}}(x_0, u_1)\, \dif u_1$.

\begin{proposition}[Action channel]\label{prop:vbp-r-ux}
    At stationarity:
    \begin{equation}\label{eq:vbp-r-ux-stationary}
        r_{u|x}^*(u_1 | x_0) = \frac{q_{\mathrm{pair}}(x_0, u_1)}{q_{x_0}(x_0)} = q(u_1 | x_0),
    \end{equation}
    the conditional from the dynamics factor belief.
\end{proposition}

\subsubsection{Singleton stationarity}

The singleton stationarity conditions are identical to the corresponding degree-counting argument in \refappx{appx:combined_detailed_stationarity}: degree-2 multipliers are identified algebraically, while degree-3 multipliers satisfy the same constraint equations on sums of multipliers.

\subsection{Message-Passing Equations}
\label{appx:vbp_messages}

Using the observation kernel $\tilde{f}_{\mathrm{obs}} = p(y_1 | x_1, \theta)$ and the VBP dynamics kernel $\tilde{f}_{\mathrm{dyn}} = p(x_1 | x_0, \theta, u_1)\, r_{u|x}(u_1 | x_0)$, the sum-product messages follow from the standard factor-to-variable update. Variable-to-factor messages $\mu_{i \to a}$ are the usual products of all incoming factor-to-variable messages except the one from the recipient factor.

\subsubsection{Messages from Observation Factor}

Standard sum-product (no channel modification):
\begin{subequations}\label{eq:vbp-msgs-obs}
    \begin{align}
        \mu_{\mathrm{obs} \to \theta}(\theta) & = \iint p(y_1 | x_1, \theta)\, \mu_{y_1 \to \mathrm{obs}}(y_1)\, \mu_{x_1 \to \mathrm{obs}}(x_1)\, \dif y_1\, \dif x_1, \label{eq:vbp-obs-theta}       \\
        \mu_{\mathrm{obs} \to x_1}(x_1)       & = \iint p(y_1 | x_1, \theta)\, \mu_{y_1 \to \mathrm{obs}}(y_1)\, \mu_{\theta \to \mathrm{obs}}(\theta)\, \dif y_1\, \dif \theta, \label{eq:vbp-obs-x1} \\
        \mu_{\mathrm{obs} \to y_1}(y_1)       & = \iint p(y_1 | x_1, \theta)\, \mu_{\theta \to \mathrm{obs}}(\theta)\, \mu_{x_1 \to \mathrm{obs}}(x_1)\, \dif x_1\, \dif \theta. \label{eq:vbp-obs-y1}
    \end{align}
\end{subequations}

\subsubsection{Messages from Dynamics Factor}

Using kernel $p(x_1 | x_0, \theta, u_1)\, r_{u|x}(u_1 | x_0)$:
\begin{subequations}\label{eq:vbp-msgs-dyn}
    \begin{align}
        \mu_{\mathrm{dyn} \to \theta}(\theta) & = \iiint p(x_1 | x_0, \theta, u_1)\, r_{u|x}(u_1 | x_0)\, \mu_{x_0 \to \mathrm{dyn}}(x_0)\, \mu_{u_1 \to \mathrm{dyn}}(u_1)\, \mu_{x_1 \to \mathrm{dyn}}(x_1)\, \dif x_1\, \dif x_0\, \dif u_1, \label{eq:vbp-dyn-theta}       \\
        \mu_{\mathrm{dyn} \to x_1}(x_1)       & = \iiint p(x_1 | x_0, \theta, u_1)\, r_{u|x}(u_1 | x_0)\, \mu_{x_0 \to \mathrm{dyn}}(x_0)\, \mu_{u_1 \to \mathrm{dyn}}(u_1)\, \mu_{\theta \to \mathrm{dyn}}(\theta)\, \dif x_0\, \dif u_1\, \dif \theta, \label{eq:vbp-dyn-x1} \\
        \mu_{\mathrm{dyn} \to x_0}(x_0)       & = \iiint p(x_1 | x_0, \theta, u_1)\, r_{u|x}(u_1 | x_0)\, \mu_{x_1 \to \mathrm{dyn}}(x_1)\, \mu_{u_1 \to \mathrm{dyn}}(u_1)\, \mu_{\theta \to \mathrm{dyn}}(\theta)\, \dif x_1\, \dif u_1\, \dif \theta, \label{eq:vbp-dyn-x0} \\
        \mu_{\mathrm{dyn} \to u_1}(u_1)       & = \iiint p(x_1 | x_0, \theta, u_1)\, r_{u|x}(u_1 | x_0)\, \mu_{x_0 \to \mathrm{dyn}}(x_0)\, \mu_{\theta \to \mathrm{dyn}}(\theta)\, \mu_{x_1 \to \mathrm{dyn}}(x_1)\, \dif x_1\, \dif x_0\, \dif \theta. \label{eq:vbp-dyn-u1}
    \end{align}
\end{subequations}

\subsubsection{Factor Beliefs and Channel Update}

The factor beliefs are the kernel times all incoming variable-to-factor messages:
\begin{subequations}\label{eq:vbp-factor-beliefs}
    \begin{align}
        q_{\mathrm{obs}}^*(y_1, x_1, \theta)      & \propto p(y_1 | x_1, \theta)\, \mu_{y_1 \to \mathrm{obs}}(y_1)\, \mu_{x_1 \to \mathrm{obs}}(x_1)\, \mu_{\theta \to \mathrm{obs}}(\theta), \label{eq:vbp-q-obs}                                                             \\
        q_{\mathrm{dyn}}^*(x_1, x_0, \theta, u_1) & \propto p(x_1 | x_0, \theta, u_1)\, r_{u|x}(u_1 | x_0)\, \mu_{x_1 \to \mathrm{dyn}}(x_1)\, \mu_{x_0 \to \mathrm{dyn}}(x_0)\, \mu_{\theta \to \mathrm{dyn}}(\theta)\, \mu_{u_1 \to \mathrm{dyn}}(u_1). \label{eq:vbp-q-dyn}
    \end{align}
\end{subequations}

The channel update follows from \refprop{prop:vbp-r-ux}:
\begin{equation}\label{eq:vbp-channel-update}
    r_{u|x}^*(u_1 | x_0) = \frac{q_{\mathrm{pair}}(x_0, u_1)}{q_{x_0}(x_0)} = q(u_1 | x_0), \quad \text{where } q_{\mathrm{pair}}(x_0, u_1) = \iint q_{\mathrm{dyn}}^*\, \dif x_1\, \dif \theta.
\end{equation}

\subsubsection{Singleton Beliefs}

Singleton beliefs follow the standard sum-product rule: the product of all incoming factor-to-variable messages with the prior (or goal prior):
\begin{subequations}\label{eq:vbp-singletons}
    \begin{align}
        q_\theta^*(\theta) & \propto p(\theta)\, \mu_{\mathrm{obs} \to \theta}(\theta)\, \mu_{\mathrm{dyn} \to \theta}(\theta), \label{eq:vbp-q-theta} \\
        q_{x_1}^*(x_1)     & \propto \hat{p}_x(x_1)\, \mu_{\mathrm{obs} \to x_1}(x_1)\, \mu_{\mathrm{dyn} \to x_1}(x_1), \label{eq:vbp-q-x1}           \\
        q_{x_0}^*(x_0)     & \propto p(x_0)\, \mu_{\mathrm{dyn} \to x_0}(x_0), \label{eq:vbp-q-x0}                                                     \\
        q_{u_1}^*(u_1)     & \propto p(u_1)\, \mu_{\mathrm{dyn} \to u_1}(u_1), \label{eq:vbp-q-u1}                                                     \\
        q_{y_1}^*(y_1)     & \propto \hat{p}_y(y_1)\, \mu_{\mathrm{obs} \to y_1}(y_1). \label{eq:vbp-q-y1}
    \end{align}
\end{subequations}

\subsection{Generic Scheme for Arbitrary \texorpdfstring{$T$}{T}}
\label{appx:vbp_generic_scheme}

The VBP scheme generalizes to arbitrary time horizons by introducing a time-local channel $r_{u|x,t}(u_t | x_{t-1})$ at each time step $t = 1, \ldots, T$. As in \refappx{appx:combined_detailed_generic}, the per-timestep additivity of the Lagrangian and the entropy correction $+\sum_t \ent{q(u_t | x_{t-1})}$ yields the multi-step scheme directly from the $T{=}1$ derivation.

\subsubsection{Messages from Observation Factor \texorpdfstring{$f_{\mathrm{obs}_t}$}{fobs,t}}

Standard sum-product (no channel modification):
\begin{subequations}\label{eq:vbp-gen-msgs-obs}
    \begin{align}
        \mu_{\mathrm{obs}_t \to \theta}(\theta) & = \iint p(y_t | x_t, \theta)\, \mu_{y_t \to \mathrm{obs}_t}(y_t)\, \mu_{x_t \to \mathrm{obs}_t}(x_t)\, \dif y_t\, \dif x_t, \label{eq:vbp-gen-obs-theta}       \\
        \mu_{\mathrm{obs}_t \to x_t}(x_t)       & = \iint p(y_t | x_t, \theta)\, \mu_{y_t \to \mathrm{obs}_t}(y_t)\, \mu_{\theta \to \mathrm{obs}_t}(\theta)\, \dif y_t\, \dif \theta, \label{eq:vbp-gen-obs-xt} \\
        \mu_{\mathrm{obs}_t \to y_t}(y_t)       & = \iint p(y_t | x_t, \theta)\, \mu_{\theta \to \mathrm{obs}_t}(\theta)\, \mu_{x_t \to \mathrm{obs}_t}(x_t)\, \dif x_t\, \dif \theta. \label{eq:vbp-gen-obs-yt}
    \end{align}
\end{subequations}

\subsubsection{Messages from Dynamics Factor \texorpdfstring{$f_{\mathrm{dyn}_t}$}{fdyn,t}}

Using kernel $p(x_t | x_{t-1}, \theta, u_t)\, r_{u|x,t}(u_t | x_{t-1})$:
\begin{subequations}\label{eq:vbp-gen-msgs-dyn}
    \begin{align}
        \mu_{\mathrm{dyn}_t \to \theta}(\theta)   & = \iiint p(x_t | x_{t-1}, \theta, u_t)\, r_{u|x,t}(u_t | x_{t-1})\, \mu_{x_{t-1} \to \mathrm{dyn}_t}(x_{t-1}) \notag                                                                                                                                          \\
                                                  & \qquad \times \mu_{u_t \to \mathrm{dyn}_t}(u_t)\, \mu_{x_t \to \mathrm{dyn}_t}(x_t)\, \dif x_t\, \dif x_{t-1}\, \dif u_t, \label{eq:vbp-gen-dyn-theta}                                                                                                         \\
        \mu_{\mathrm{dyn}_t \to x_t}(x_t)         & = \iiint p(x_t | x_{t-1}, \theta, u_t)\, r_{u|x,t}(u_t | x_{t-1})\, \mu_{x_{t-1} \to \mathrm{dyn}_t}(x_{t-1})\, \mu_{u_t \to \mathrm{dyn}_t}(u_t)\, \mu_{\theta \to \mathrm{dyn}_t}(\theta)\, \dif x_{t-1}\, \dif u_t\, \dif \theta, \label{eq:vbp-gen-dyn-xt} \\
        \mu_{\mathrm{dyn}_t \to x_{t-1}}(x_{t-1}) & = \iiint p(x_t | x_{t-1}, \theta, u_t)\, r_{u|x,t}(u_t | x_{t-1})\, \mu_{x_t \to \mathrm{dyn}_t}(x_t)\, \mu_{u_t \to \mathrm{dyn}_t}(u_t)\, \mu_{\theta \to \mathrm{dyn}_t}(\theta)\, \dif x_t\, \dif u_t\, \dif \theta, \label{eq:vbp-gen-dyn-xtm1}           \\
        \mu_{\mathrm{dyn}_t \to u_t}(u_t)         & = \iiint p(x_t | x_{t-1}, \theta, u_t)\, r_{u|x,t}(u_t | x_{t-1})\, \mu_{x_{t-1} \to \mathrm{dyn}_t}(x_{t-1})\, \mu_{\theta \to \mathrm{dyn}_t}(\theta)\, \mu_{x_t \to \mathrm{dyn}_t}(x_t)\, \dif x_t\, \dif x_{t-1}\, \dif \theta. \label{eq:vbp-gen-dyn-ut}
    \end{align}
\end{subequations}

\subsubsection{Factor Beliefs}

\begin{subequations}\label{eq:vbp-gen-factor-beliefs}
    \begin{align}
        q_{\mathrm{obs},t}^*(y_t, x_t, \theta)          & \propto p(y_t | x_t, \theta)\, \mu_{y_t \to \mathrm{obs}_t}(y_t)\, \mu_{x_t \to \mathrm{obs}_t}(x_t)\, \mu_{\theta \to \mathrm{obs}_t}(\theta), \label{eq:vbp-gen-q-obs}                                                                                 \\
        q_{\mathrm{dyn},t}^*(x_t, x_{t-1}, \theta, u_t) & \propto p(x_t | x_{t-1}, \theta, u_t)\, r_{u|x,t}(u_t | x_{t-1})\, \mu_{x_t \to \mathrm{dyn}_t}(x_t)\, \mu_{x_{t-1} \to \mathrm{dyn}_t}(x_{t-1})\, \mu_{\theta \to \mathrm{dyn}_t}(\theta)\, \mu_{u_t \to \mathrm{dyn}_t}(u_t). \label{eq:vbp-gen-q-dyn}
    \end{align}
\end{subequations}

\subsubsection{Channel Updates}

Each time step has its own channel update:
\begin{equation}\label{eq:vbp-gen-channel-update}
    r_{u|x,t}^*(u_t | x_{t-1}) = q_t(u_t | x_{t-1}), \quad \text{where } q_t(u_t | x_{t-1}) = \frac{q_{\mathrm{pair},t}(x_{t-1}, u_t)}{q_{x_{t-1}}(x_{t-1})}\,,
\end{equation}
with $q_{\mathrm{pair},t}(x_{t-1}, u_t) = \iint q_{\mathrm{dyn},t}^*\, \dif x_t\, \dif \theta$.

\subsubsection{Singleton Beliefs}

\begin{subequations}\label{eq:vbp-gen-singletons}
    \begin{align}
        q_{x_t}^*(x_t)     & \propto \hat{p}_x(x_t)\, \mu_{\mathrm{obs}_t \to x_t}(x_t)\, \mu_{\mathrm{dyn}_t \to x_t}(x_t)\, \mu_{\mathrm{dyn}_{t+1} \to x_t}(x_t), \label{eq:vbp-gen-q-xt} \\
        q_\theta^*(\theta) & \propto p(\theta) \prod_{\tau=1}^{T} \mu_{\mathrm{obs}_\tau \to \theta}(\theta)\, \mu_{\mathrm{dyn}_\tau \to \theta}(\theta), \label{eq:vbp-gen-q-theta}        \\
        q_{u_t}^*(u_t)     & \propto p(u_t)\, \mu_{\mathrm{dyn}_t \to u_t}(u_t), \label{eq:vbp-gen-q-ut}                                                                                     \\
        q_{y_t}^*(y_t)     & \propto \hat{p}_y(y_t)\, \mu_{\mathrm{obs}_t \to y_t}(y_t). \label{eq:vbp-gen-q-yt}
    \end{align}
\end{subequations}

The boundary conditions follow the usual temporal-edge conventions: at $t=1$, $\mu_{x_0 \to \mathrm{dyn}_1}(x_0) = p(x_0)$; at $t=T$, the message $\mu_{\mathrm{dyn}_{T+1} \to x_T}$ is absent.

\subsection{Properties}
\label{appx:vbp_properties}

\paragraph{No min-max structure.}
Unlike AIF, where opposing channels create a min-max optimization over beliefs, VBP has a single channel that enters only the dynamics kernel numerator. The joint optimization over $(q, r_{u|x})$ is a pure minimization problem, avoiding the convergence difficulties of AIF's saddle-point structure.

\paragraph{Convergence.}
The observation factor uses standard sum-product messages with no channel modification, so standard BP convergence guarantees apply to the observation side. The dynamics channel provides a single-sided correction. While damping may improve convergence in practice, the lack of opposing channels makes it less critical than for AIF.

\paragraph{Reduction.}
Setting $r_{u|x,t}$ to uniform for all $t$ removes the entropy correction and recovers standard belief propagation (marginal inference). The VBP scheme thus continuously interpolates between marginal inference (uniform channel) and full cross-entropy planning (converged channel).

\paragraph{Fixed-point interpretation.}
At convergence, $r_{u|x,t}^*(u_t | x_{t-1}) = q(u_t | x_{t-1})$, so the VBP dynamics kernel becomes $p(x_t | x_{t-1}, \theta, u_t)\, q(u_t | x_{t-1})$. This reweights the state transition by the policy conditioned on the previous state, implementing the action commitment that cross-entropy planning encodes.

\section{Experiment Details}
\label{appx:experiments}

This appendix provides full details for the experiments in \capsecref{sec:experiments}.

\subsection{Frozen Lake Environment}
\label{appx:frozen_lake}

\paragraph{Grid layout.}
We adapt the classic Frozen Lake environment~\citep{brockman_openai_2016,towers_gymnasium_2025} to include epistemic uncertainty by treating the hole layout as unknown.
A $4{\times}4$ grid where the agent starts at the top-left cell and must reach the goal at the bottom-right cell.
A subset of the remaining cells are holes; stepping into a hole terminates the episode with failure.
The ice surface makes transitions stochastic: with probability $1 - p_\text{slip}$ the agent moves in the intended direction, and with probability $p_\text{slip}/3$ it slips to each of the three remaining directions, where $p_\text{slip} = 0.1$.

\paragraph{Configurations.}
The hole layout is the unknown parameter~$\theta$.
We sample $15$ hole configurations uniformly at random (with fixed seeds for reproducibility), each placing $2$ holes on the grid (a fraction $0.2$ of the $14$ non-start/goal cells, truncated to an integer).
The agent does not know the hole locations and must infer them from observations.

\paragraph{Observation model.}
The observation model has two modalities.
First, $2 n_\text{pos}$ position channels observe the agent's position and scan mode with near-deterministic precision ($0.999/0.001$), so the agent always knows where it is.
Second, $n_\text{pos}$ grid-cell channels each provide a binary ``hole/safe'' reading for the corresponding cell.
Unscanned grid-cell observations are corrupted by distance-dependent noise: $\text{noise} = \alpha_\text{base} + \alpha_\text{range} \cdot d / d_\text{max}$, where $d$ is the Manhattan distance from the agent to the cell.
A low-noise reading directly constrains which configurations~$\theta$ are consistent, making the observation model approximately unambiguous.
A SCAN action switches all grid-cell observations to near-deterministic ($0.999/0.001$) at the cost of one time step.

\paragraph{State, action, and observation spaces.}
States encode position and scan mode: $2 \times n_\text{pos}$ states total (e.g., $32$ for a $4{\times}4$ grid).
Actions are the four cardinal directions plus SCAN ($|\mathcal{U}| = 5$).
Observations consist of $2 n_\text{pos}$ position channels and $n_\text{pos}$ grid-cell channels.

\paragraph{Priors.}
The goal prior $\hat{p}(x_T)$ is a softmax preference peaking at the bottom-right cell, with penalties for hole positions (varying per~$\theta$).
The parameter prior is uniform over the $15$ configurations: $p(\theta) = 1/15$.
The initial state prior $p(x_0)$ places all mass on the top-left cell in unscanned mode.
The action prior assigns weight~$1$ to each movement action and weight $c_\text{scan} = 0.1$ to SCAN, then normalizes: $p(u_t) = w_{u_t} / \sum_u w_u$, giving $p(\text{move}) \approx 0.244$ and $p(\text{SCAN}) \approx 0.024$.

\paragraph{Planning parameters.}
Planning horizon $T = 15$, fixed across all decision steps.
All methods use $400$ iterations.
We run $1000$ episodes per method with a maximum of $15$ steps per episode.
Episode~$i$ uses seed $i$ for reproducibility.

\subsection{RockSample Environment}
\label{appx:rocksample}

\paragraph{Grid layout.}
We use the canonical RockSample~$(4,3)$ environment~\citep{smith_heuristic_2004} in the epistemic planning framework, treating rock quality as the unknown parameter~$\theta$.
A $4{\times}4$ grid where the agent starts at the left edge and can exit via the right edge.
Three rocks are placed at known grid positions; their quality (good or bad) is unknown.

\paragraph{Configurations.}
Rock quality defines the unknown parameter~$\theta$.
With $3$ rocks each having binary quality, there are $n_\theta = 8$ configurations.
The agent does not know rock quality and must infer it from CHECK readings.

\paragraph{Observation model.}
The observation model has two components.
First, position channels observe the agent's position with noise parameter $\alpha_\text{pos} = 0.1$.
Second, rock-quality information is available only through CHECK actions: a CHECK of rock~$i$ returns a binary ``good/bad'' reading whose accuracy depends on the Euclidean distance~$d$ from the agent to that rock:
$p(\text{correct} \mid d) = \tfrac{1}{2}\bigl(1 + 2^{-d/d_{1/2}}\bigr)$,
where $d_{1/2} = 2$ is the half-efficiency distance.
At $d = 0$ the reading is deterministic; at $d = d_{1/2}$ accuracy is $75\%$; as $d \to \infty$ the reading approaches chance.
Outside CHECK actions the agent receives no rock information, so no action changes the observation model at future steps and novelty is the only source of epistemic value.

\paragraph{Actions and rewards.}
The agent has nine actions: four cardinal movements, one CHECK per rock, SAMPLE, and EXIT (move off the right edge).
CHECK$_i$ senses rock~$i$ at the cost of one time step, providing the explicit epistemic actions.
SAMPLE collects the rock at the current position, yielding reward $+2$ (good rock) or penalty $-3$ (bad rock).
EXIT gives a fixed reward $+1$.
Movement is deterministic ($p_\text{slip} = 0$).

\paragraph{State, action, and observation spaces.}
States encode position: $n_\text{pos} = 16$ states for the $4{\times}4$ grid.
Actions: four cardinal directions, three CHECK actions, SAMPLE, and EXIT ($|\mathcal{U}| = 9$).
Observations consist of $n_\text{pos}$ position channels and $3$ binary rock-quality channels.

\paragraph{Priors.}
The goal prior $\hat{p}(x_T)$ is a softmax preference with temperature $\tau = 1$ peaking at EXIT cells, with penalties for remaining on the grid.
The parameter prior is uniform over the $8$ configurations: $p(\theta) = 1/8$.
The action prior assigns weight~$1$ to each movement action and weight $c_\text{exit} = 0.5$ to EXIT, then normalizes.

\paragraph{Planning parameters.}
Planning horizon $T = 12$, fixed across all decision steps.
All methods use $100$ iterations.
We run $1000$ episodes per method with a maximum of $25$ steps per episode.
Episode~$i$ uses seed $i$ for reproducibility.

\paragraph{Full results with confidence intervals.}
\Cref{tab:rocksample_full} reports RockSample results with 95\% confidence intervals.

\begin{table}[ht]
    \centering
    \caption{RockSample results with 95\% confidence intervals, averaged over 1000 episodes.}
    \label{tab:rocksample_full}
    \setlength{\tabcolsep}{4pt}
    \begin{tabular}{l c c c}
        \toprule
        \textbf{Method} & \textbf{Avg.\ reward} & \textbf{Retrieval (\%)} & \textbf{Avg.\ steps} \\
        \midrule
        BP              & $1.00\;[1.00, 1.00]$  & $0.0$                   & $3.00$               \\
        VBP             & $1.00\;[1.00, 1.00]$  & $0.0$                   & $3.00$               \\
        RM-MP           & $1.00\;[1.00, 1.00]$  & $0.0$                   & $3.00$               \\
        Nuijten-MP      & $1.00\;[1.00, 1.00]$  & $0.0$                   & $3.00$               \\
        AIF-MP          & $4.01\;[3.90, 4.12]$  & $98.7$                  & $8.56$               \\
        \bottomrule
    \end{tabular}
\end{table}

\subsection{Wumpus World Environment}
\label{appx:wumpus_world}

\paragraph{Grid layout.}
We adapt the classic Wumpus World environment~\citep{russell_artificial_1995} to include epistemic uncertainty by treating the layout as unknown.
We simplify the classic dynamics to isolate the epistemic challenge: the agent has no orientation or inventory and navigates by cardinal movement.
Transitions slip to one of the three unintended directions with probability $p_\text{slip} = 0.01$, as in Frozen Lake.
The agent starts at cell~$0$ and must reach the gold cell.
The grid contains pits and a wumpus (both absorbing hazards that terminate the episode) and a single gold cell.

\paragraph{Configurations.}
The locations of pits, wumpus, and gold define the unknown parameter~$\theta$.
We use $25$ configurations sampled with fixed seeds.
Each configuration places $4$ pits, one wumpus, and one gold on the grid, excluding the agent's starting cell.

\paragraph{Observation model.}
The observation model has two components, both noisy when unscanned.
Three binary feature channels detect adjacency to hazards: \emph{breeze} fires if adjacent to a pit, \emph{stench} fires if adjacent to the wumpus, and \emph{glitter} fires if on the gold cell.
Unscanned feature channels have true-positive probability $p_\text{tp} = 1 - \alpha_\text{obs}$ and false-positive probability $p_\text{fp} = 0.1 \, \alpha_\text{obs}$.
Additionally, $n_\text{pos}$ position channels encode the agent's position, also noisy when unscanned.
Observations are ambiguous: a breeze indicates a nearby pit but not which neighbor, and position uncertainty compounds this ambiguity.
The agent must integrate evidence across multiple positions to triangulate hazard locations.
A SCAN action switches the three feature channels to near-deterministic ($0.999/0.001$) for a single step, at the cost of one time step; the scan mode decays after one step, and position channels are unaffected.
Because the improved precision lasts only one step, scanning does not change the observation model at future steps: the value of a scan lies entirely in the information the resulting reading carries about~$\theta$.

\paragraph{State, action, and observation spaces.}
States encode position and a transient scan mode that decays after one step: $2 \times n_\text{pos}$ states total (e.g., $50$ for a $5{\times}5$ grid).
Actions are the four cardinal directions plus SCAN ($|\mathcal{U}| = 5$).
Observations consist of $3$ binary feature channels (breeze, stench, glitter) and $n_\text{pos}$ position channels.

\paragraph{Priors.}
The goal prior $\hat{p}(x_T)$ is a softmax preference peaking at the gold cell for each~$\theta$, with penalties for pits and the wumpus.
The parameter prior is uniform over the $25$ configurations.
The initial state prior $p(x_0)$ places all mass on position~$0$ in unscanned mode.
The action prior assigns weight~$1$ to each movement action and weight $c_\text{scan} = 0.7$ to SCAN, then normalizes: $p(u_t) = w_{u_t} / \sum_u w_u$, giving $p(\text{move}) \approx 0.213$ and $p(\text{SCAN}) \approx 0.149$.

\paragraph{Planning parameters.}
Planning horizon $T = 8$, fixed across all decision steps.
All methods use $150$ iterations.
SCAN costs one time step (same as Frozen Lake).
We run $1000$ episodes per method with a maximum of $16$ steps per episode.
Episode~$i$ uses seed $i$ for reproducibility.

\paragraph{Representative trajectories.}
\label{appx:wumpus_trajectories}
\Cref{fig:wumpus_trajectory_16,fig:wumpus_trajectory_02} show two episodes that illustrate the behavioral signatures behind the aggregate Wumpus gap reported in \capsecref{subsec:results}.
Each panel shows one method on a fixed layout, with the agent's path shaded by step order (light early, dark late) and SCAN actions marked at the cell where they were taken.
The line below each panel reports the terminal reward, the number of steps, and the outcome.

BP and VBP never scan and stall near the start in both episodes.
RM-MP moves more but without direction: a single early scan does not redirect it, and it times out mid-grid (\Cref{fig:wumpus_trajectory_16}) or falls into a pit (\Cref{fig:wumpus_trajectory_02}).
Nuijten-MP scans but does not convert the evidence into safe progress: it advances only one cell before timing out (\Cref{fig:wumpus_trajectory_16}) or walks into the pit beside the start (\Cref{fig:wumpus_trajectory_02}).
AIF-MP interleaves single scans with movement and is the only method that reaches the gold in either episode.
Aggregate success rates nevertheless remain below half (\Cref{tab:results}); these episodes illustrate the qualitative separation, not uniform success.

\begin{figure}[ht]
    \centering
    \includegraphics[width=\textwidth]{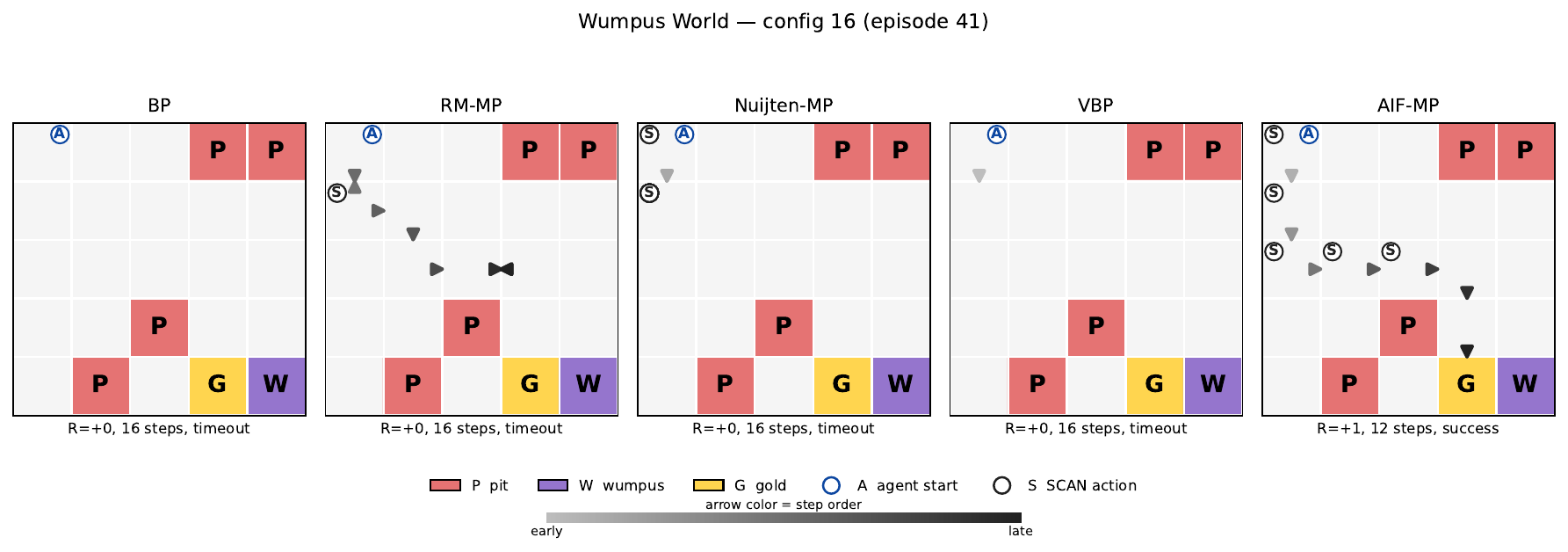}
    \caption{Wumpus World trajectories for all five methods on configuration~$16$, episode~$41$. Symbols: P pit, W wumpus, G gold, A agent start, S SCAN action; arrow shade encodes step order. BP and VBP never scan and stall near the start. RM-MP scans once but wanders to the middle of the grid and times out. Nuijten-MP scans twice yet advances only one cell. AIF-MP interleaves scans with movement and reaches the gold in $12$ steps.}
    \label{fig:wumpus_trajectory_16}
\end{figure}

\begin{figure}[ht]
    \centering
    \includegraphics[width=\textwidth]{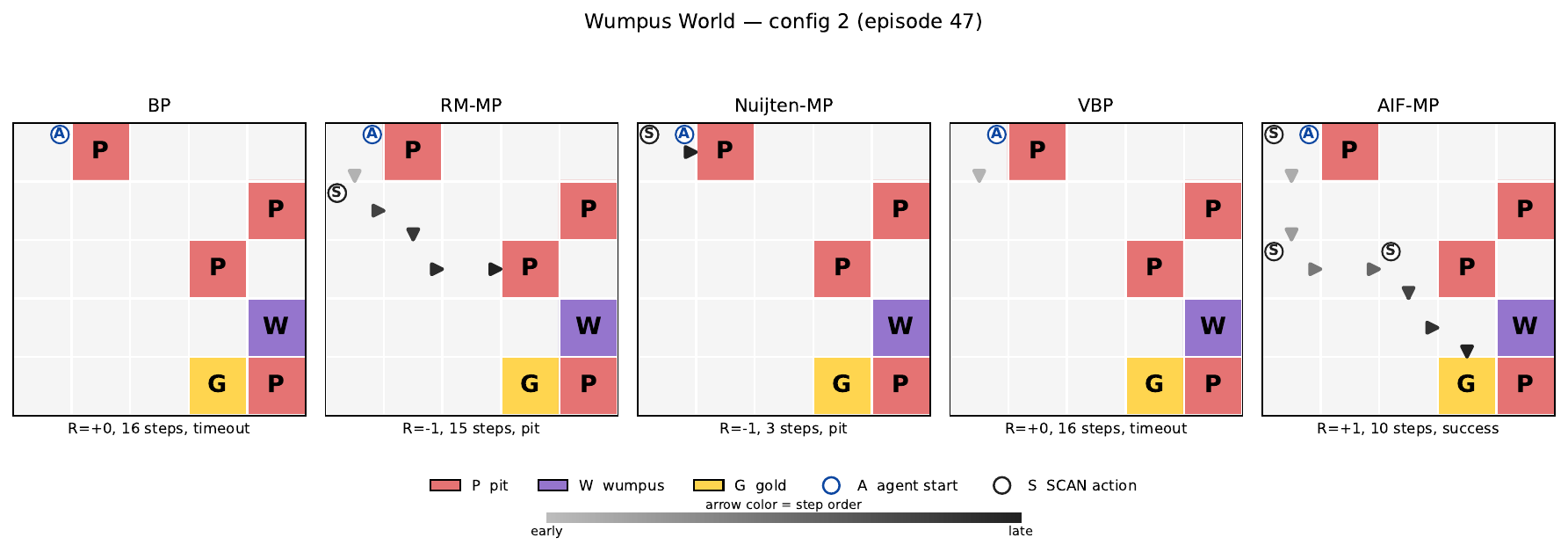}
    \caption{Wumpus World trajectories on configuration~$2$, episode~$47$, where the gold lies in the opposite corner of the grid. BP and VBP stall near the start. RM-MP wanders and falls into a pit after $15$ steps. Nuijten-MP scans once and then walks into the pit beside the start. AIF-MP scans its way across the grid and reaches the gold in $10$ steps.}
    \label{fig:wumpus_trajectory_02}
\end{figure}

\subsection{Common Implementation Details}
\label{appx:implementation}

\paragraph{Software framework.}
All tensor operations and message-passing routines use JAX \citep{bradbury_jax_2018} with JIT compilation.
All planning and inference functions are decorated with \texttt{jax.jit}, with the planning horizon and number of iterations as compile-time constants.

\paragraph{Log-space computation.}
All internal messages, beliefs, and channels are stored as log-probabilities.
A sentinel value of $-10^{12}$ replaces $-\infty$ for numerical stability, and a safe logarithm floors its argument at $10^{-30}$ before taking the log.
Conversion to probability space occurs only at final output via softmax.

\paragraph{Message damping.}
Channel-based methods (VBP, RM-MP, and AIF-MP) apply geometric damping~\eqref{eq:damping}, implemented as a linear interpolation in log-space followed by per-condition renormalization.
Damping is applied to every channel the method uses (policy, dynamics, observation, and marginal observation) after each BP iteration.
Structural zeros are preserved: a position remains at $-10^{12}$ only if both old and new values are $-10^{12}$.
For VBP, this reduces to damping the single policy channel; BP and Nuijten-MP do not use damping ($\lambda = 1.0$).
The damping parameter~$\lambda$ is selected per method and environment from the convergence sweep described in \Cref{appx:convergence}.
\Cref{tab:damping_values} reports the selected values.

\begin{table}[ht]
    \centering
    \caption{Damping parameter $\lambda$ per method and environment, selected from the convergence sweep (\Cref{appx:convergence}).}
    \label{tab:damping_values}
    \setlength{\tabcolsep}{4pt}
    \begin{tabular}{l c c c}
        \toprule
        \textbf{Method} & \textbf{Frozen Lake} & \textbf{RockSample} & \textbf{Wumpus World} \\
        \midrule
        BP              & $1.0$                & $1.0$               & $1.0$                 \\
        VBP             & $0.9$                & $0.9$               & $0.9$                 \\
        RM-MP           & $0.25$               & $0.25$              & $0.25$                \\
        Nuijten-MP      & $1.0$                & $1.0$               & $1.0$                 \\
        AIF-MP          & $0.9$                & $0.9$               & $0.75$                \\
        \bottomrule
    \end{tabular}
\end{table}

\paragraph{Message initialization.}
All channels (dynamics, observation, and marginal observation) are initialized to uniform distributions over their respective domains.
Parameter cavity beliefs are initialized to the prior $p(\theta)$; state beliefs are initialized to uniform over valid states.

\paragraph{Action selection.}
Actions are selected by greedy argmax over the action marginal at $t = 0$ (deterministic, no sampling).

\subsection{Convergence Behavior}
\label{appx:convergence}

To select damping parameters and characterize convergence, we run a systematic sweep over $\lambda \in \{0.25, 0.4, 0.5, 0.6, 0.75, 0.9\}$ for each channel-based method (RM-MP, VBP, AIF-MP) and $\lambda = 1.0$ for loopy BP, across all three environments.
Each configuration is run with $5$ random seeds and $1{,}000$ BP iterations.
Convergence is declared when the maximum absolute change in any channel falls below $10^{-4}$.

\paragraph{Convergence rate.}
\Cref{fig:convergence_heatmaps} reports the fraction of seeds that converge (color) and the median number of iterations to convergence (in parentheses) for each method--damping combination.
The three environments exhibit qualitatively different convergence profiles.

On RockSample (\Cref{fig:conv_heatmap_rs}), all methods converge at all damping values.
The deterministic dynamics make the dynamics channel update exact: it converges within two iterations regardless of~$\lambda$, so the damping choice is immaterial for RM-MP here.

On Frozen Lake (\Cref{fig:conv_heatmap_fl}), the dynamics channel (RM-MP) is the most fragile: it converges only at conservative damping ($80\%$ for $\lambda \leq 0.4$) and not at all for $\lambda \geq 0.5$, with large VFE oscillations.
This is consistent with the min-max structure identified in \capsecref{sec:message_passing}: the stochastic dynamics activate the opposing-sign dynamics channel, which amplifies update steps when damping is insufficient.
VBP is stable at nearly all damping values ($60$--$100\%$), since it uses only the planning channel (no opposing signs).
AIF-MP converges reliably at higher damping ($100\%$ at $\lambda = 0.9$) but slowly at conservative settings ($20\%$ at $\lambda = 0.25$).

On Wumpus World (\Cref{fig:conv_heatmap_ww}), VBP and AIF-MP converge reliably: VBP at $100\%$ for every damping value, AIF-MP at $100\%$ for $\lambda \leq 0.75$ and $80\%$ at $\lambda = 0.9$.
The dynamics channel is again the most fragile, converging on at most $40\%$ of seeds and on none for $\lambda \geq 0.6$.
The one-step scan and the local adjacency signals keep the observation-side channels active, yet the joint scheme remains stable.

\begin{figure}[ht]
    \centering
    \begin{subfigure}[t]{0.48\textwidth}
        \centering
        \includegraphics[width=\textwidth]{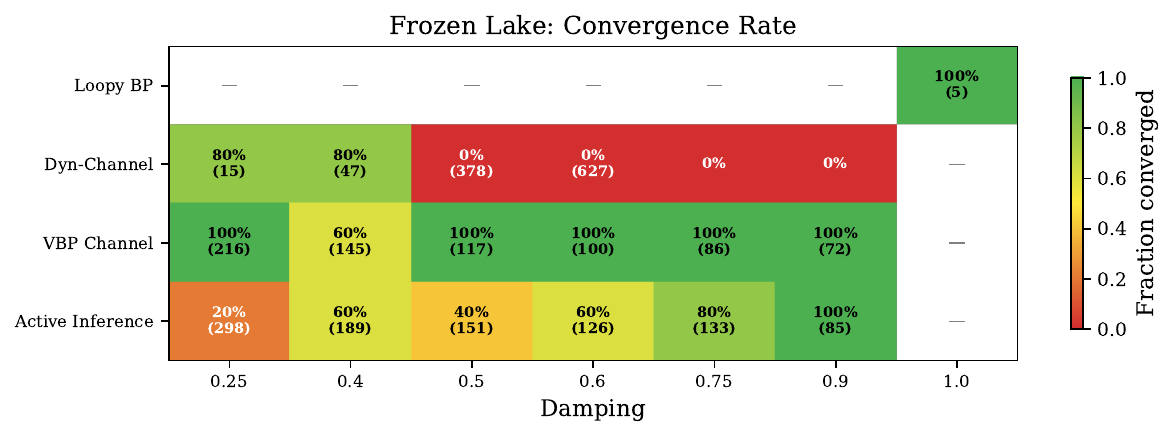}
        \caption{Frozen Lake}
        \label{fig:conv_heatmap_fl}
    \end{subfigure}
    \hfill
    \begin{subfigure}[t]{0.48\textwidth}
        \centering
        \includegraphics[width=\textwidth]{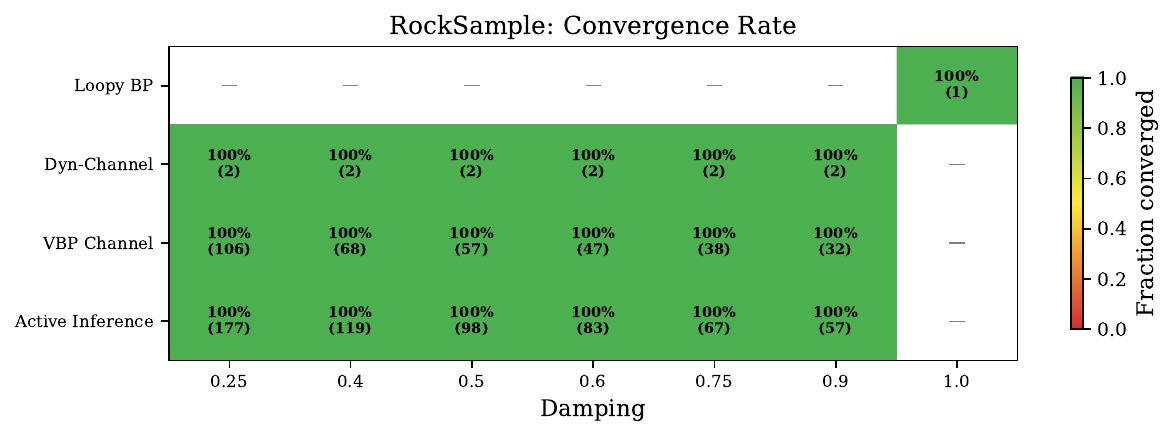}
        \caption{RockSample}
        \label{fig:conv_heatmap_rs}
    \end{subfigure}
    \\[1em]
    \begin{subfigure}[t]{0.48\textwidth}
        \centering
        \includegraphics[width=\textwidth]{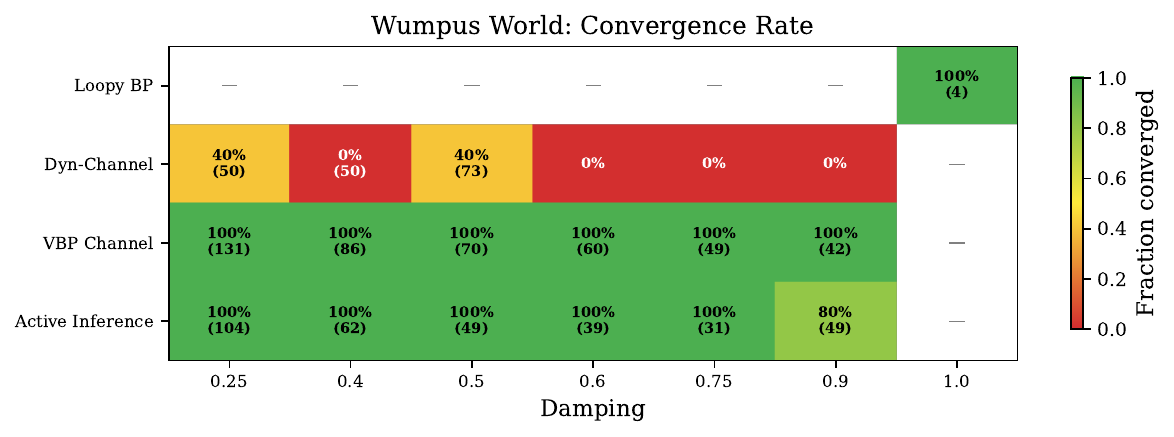}
        \caption{Wumpus World}
        \label{fig:conv_heatmap_ww}
    \end{subfigure}
    \caption{Convergence rate (color) and median iterations to convergence (in parentheses) for each method and damping value~$\lambda$, averaged over $5$ seeds with $1{,}000$ iterations each. Dashes indicate that the method does not use damping at that value.}
    \label{fig:convergence_heatmaps}
\end{figure}

\paragraph{VFE convergence dynamics.}
\Cref{fig:vfe_method_comparison} shows VFE traces for all methods at their best damping on Frozen Lake.
All four methods reach a stationary plateau within the iteration budget: loopy BP within ${\sim}5$ iterations, and the channel-augmented methods within $150$ iterations, depending on the number of active channels.
The absolute VFE values at the plateau are not directly comparable across methods because each method optimizes a different functional (\Cref{tab:method_comparison}).

\begin{figure}[ht]
    \centering
    \includegraphics[width=0.9\textwidth]{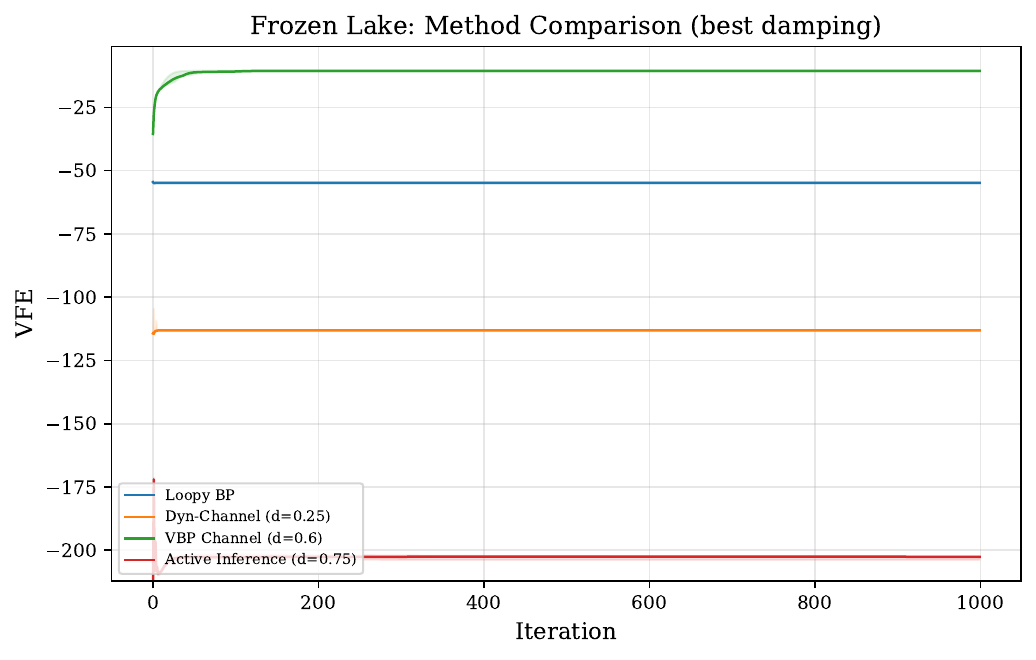}
    \caption{VFE traces on Frozen Lake for all methods at their best damping (seed-averaged with $1\sigma$ bands). All four methods reach a stationary plateau within the iteration budget.}
    \label{fig:vfe_method_comparison}
\end{figure}

\paragraph{Damping selection.}
The damping parameter~$\lambda$ is selected per method and environment as the value with the highest convergence rate; ties are broken by fewest median iterations.
The selected values are reported in \Cref{tab:damping_values}.
The pattern is consistent across environments: RM-MP requires conservative damping ($\lambda = 0.25$) due to the opposing dynamics channel, VBP tolerates aggressive damping ($\lambda = 0.9$), and AIF-MP sits in between ($\lambda = 0.75$ on Wumpus World, $0.9$ elsewhere).
The worst observed case at the selected values is a single AIF-MP seed on Wumpus World that needs ${\sim}320$ iterations; all other runs converge within $150$.
Developing convergence theory for the channel-augmented scheme remains an open problem.

\end{document}